\documentclass[11pt]{article}

\usepackage[utf8]{inputenc}
\usepackage[margin=1.0in]{geometry}
\usepackage{charter,euler}

\usepackage{amsmath,amssymb,amsthm,amsfonts,mathtools,authblk}
\usepackage{graphicx}
\usepackage{float}
\usepackage{booktabs}
\usepackage{array}
\usepackage{caption}
\usepackage[dvipsnames,svgnames]{xcolor}
\definecolor{natblue}{RGB}{0,83,156}

\usepackage[ruled,vlined,linesnumbered]{algorithm2e}

\usepackage{enumitem}
\usepackage{url}
\usepackage{xspace}
\xspaceaddexceptions{]\,}   
\usepackage{siunitx}
\usepackage{tikz}
\usetikzlibrary{positioning, calc, arrows.meta}

\usepackage[colorlinks=true,linkcolor=natblue,citecolor=natblue,urlcolor=natblue]{hyperref}

\providecommand{\keywords}[1]{\textbf{Keywords:} #1}

\newcommand{\ECS}{\textsc{ecs}\xspace}

\newcommand{\MTVFL}{MTVFL\xspace}

\newcommand{\MKL}{MKL\xspace}

\newcommand{\Rbb}{\mathbb{R}}
\DeclareMathOperator{\EC}{\chi}
\newcommand{\Lone}{L^1}
\newcommand{\norm}[1]{\left\lVert #1 \right\rVert}
\newcommand{\ie}{\textit{i.e.}\xspace}

\newcommand{\nTrials}{37\xspace}
\newcommand{\nStation}{36\xspace}
\newcommand{\ugsMin}{2.96\xspace}
\newcommand{\ugsMax}{13.94\xspace}
\newcommand{\gridStep}{0.31\xspace}
\newcommand{\ulsVal}{0.14\xspace}
\newcommand{\boreMm}{48.59\xspace}
\newcommand{\velocityFactorLo}{0.70\xspace}
\newcommand{\velocityFactorHi}{0.91\xspace}
\newcommand{\pMidLo}{14.6\xspace}
\newcommand{\pMidHi}{18.9\xspace}
\newcommand{\captureFps}{240\xspace}
\newcommand{\sampleMs}{37.5\xspace}
\newcommand{\nyquistHz}{13.3\xspace}
\newcommand{\windowS}{0.94\xspace}

\newcommand{\framesShipped}{348\xspace}
\newcommand{\scalesKept}{19\xspace}
\newcommand{\lambdaVal}{0.1\xspace}
\newcommand{\LDbot}{50\xspace}
\newcommand{\LDmid}{139\xspace}
\newcommand{\LDtop}{227\xspace}
\newcommand{\Dstar}{18\xspace}
\newcommand{\mantelEcs}{0.50\xspace}
\newcommand{\mantelAmp}{0.26\xspace}
\newcommand{\mantelEcsRe}{0.19\xspace}
\newcommand{\mantelPmax}{0.025\xspace}
\newcommand{\compRho}{0.68\xspace}
\newcommand{\compShare}{19\%\xspace}
\newcommand{\cutTwo}{9.97\xspace}
\newcommand{\cutThreeA}{7.89\xspace}
\newcommand{\cutThreeB}{11.70\xspace}
\newcommand{\cutFourA}{4.33\xspace}
\newcommand{\cutFourAlo}{4.13\xspace}
\newcommand{\cutFourAhi}{4.53\xspace}
\newcommand{\cutFourB}{7.51\xspace}
\newcommand{\cutFourBlo}{7.32\xspace}
\newcommand{\cutFourBhi}{7.70\xspace}
\newcommand{\cutFourC}{11.70\xspace}
\newcommand{\cutFourClo}{11.58\xspace}
\newcommand{\cutFourChi}{11.82\xspace}
\newcommand{\betaEcs}{0.25\xspace}
\newcommand{\betaAmp}{0.75\xspace}
\newcommand{\gapSC}{0.87\xspace}
\newcommand{\gapSCDir}{below\xspace}
\newcommand{\gapCA}{1.17\xspace}
\newcommand{\gapCADir}{below\xspace}
\newcommand{\gapCATaitel}{1.53\xspace}
\newcommand{\gapCATaitelDir}{below\xspace}
\newcommand{\nomWuSC}{5.3\xspace}
\newcommand{\nomWuCA}{12.7\xspace}

\newcommand{\bbSC}{5.2\xspace}
\newcommand{\miSC}{4.2\xspace}
\newcommand{\scSubmittedNow}{7.89\xspace}

\newcommand{\jackFourMin}{35\xspace}
\newcommand{\jackFourMax}{36\xspace}
\newcommand{\nBoot}{500\xspace}
\newcommand{\bootNearA}{23\%\xspace}
\newcommand{\bootNearB}{27\%\xspace}
\newcommand{\bootNearC}{32\%\xspace}
\newcommand{\stBotThree}{4.00 and 7.25\xspace}
\newcommand{\stMidThree}{8.95 and 11.70\xspace}
\newcommand{\stTopThree}{7.63 and 12.54\xspace}
\newcommand{\jackMidLow}{50\%\xspace}
\newcommand{\jackTopLow}{8\%\xspace}
\newcommand{\reThree}{5.95 and 12.65\xspace}
\newcommand{\reFullThree}{8.26 and 12.65\xspace}
\newcommand{\reDiffLow}{2.31\xspace}
\newcommand{\mmThree}{4.33 and 7.89\xspace}
\newcommand{\mmBeta}{0.91\xspace}
\newcommand{\lowMin}{4.00\xspace}
\newcommand{\lowMax}{8.95\xspace}
\newcommand{\upMin}{7.25\xspace}
\newcommand{\upMax}{12.65\xspace}
\newcommand{\scSubmitted}{9.08\xspace}
\newcommand{\scSubmittedFt}{10.48\xspace}
\newcommand{\kernelFloor}{0.96\xspace}
\newcommand{\betaUgsConsistent}{1\xspace}
\newcommand{\velSubmitted}{8.61\xspace}

\newcommand{\velTieA}{36--38 or 38--40 SCFM\xspace}
\newcommand{\velTieB}{60--62 or 62--64 SCFM\xspace}
\newcommand{\tamuN}{947\xspace}
\newcommand{\tamuConds}{47\xspace}
\newcommand{\tamuTau}{0.275\xspace}
\newcommand{\tamuRatio}{3.0\xspace}
\newcommand{\tamuRatioAsIs}{2.9\xspace}
\newcommand{\tamuChurnCells}{8\xspace}
\newcommand{\tamuSlugCells}{19\xspace}
\newcommand{\tamuCellRatio}{2.9\xspace}
\newcommand{\tamuCellP}{0.025\xspace}
\newcommand{\tamuCellPAsIs}{0.002\xspace}
\newcommand{\tamuBoth}{4\xspace}
\newcommand{\tamuBothHigher}{3\xspace}
\newcommand{\tamuAccSame}{91\%\xspace}
\newcommand{\tamuAccSameAsIs}{78\%\xspace}
\newcommand{\tamuCarriedSame}{82\%\xspace}
\newcommand{\tamuAccCond}{43\%\xspace}
\newcommand{\tamuAccCondAsIs}{39\%\xspace}
\newcommand{\tamuCarriedCond}{48\%\xspace}
\newcommand{\tamuAccImage}{60\%\xspace}
\newcommand{\tamuAccImageAsIs}{35\%\xspace}
\newcommand{\tamuCarriedImage}{43\%\xspace}
\newcommand{\tamuMajority}{30--32\%\xspace}
\newcommand{\tamuHonestMin}{35\%\xspace}
\newcommand{\tamuHonestMax}{60\%\xspace}
\newcommand{\spectralChanges}{17\xspace}
\newcommand{\pacBound}{1.48\xspace}
\newcommand{\malinAmc}{76--78\xspace}

\newcommand{\malinCufTwoU}{11.95\xspace}

\newcommand{\tableXXdev}{0.1\xspace}
\newcommand{\nBrackets}{36\xspace}
\newcommand{\nChangeMid}{11\xspace}
\newcommand{\avCutsScfm}{20--22, 36--38 and 62--64\xspace}

\newcommand{\chanceCuts}{12\%\xspace}
\newcommand{\chanceSlug}{8\%\xspace}
\newcommand{\slugBot}{22--24\xspace}
\newcommand{\slugTop}{18--20\xspace}
\newcommand{\lowCutBot}{20--22\xspace}
\newcommand{\lowCutTop}{22--24\xspace}

\newif\ifshowauthornotes
\showauthornotesfalse

\theoremstyle{plain}
\newtheorem{theorem}{Theorem}
\newtheorem{proposition}[theorem]{Proposition}

\theoremstyle{definition}
\newtheorem{definition}[theorem]{Definition}
\newtheorem{remark}[theorem]{Remark}

\author[3]{Brady Koenig}
\author[1]{Sushovan Majhi\footnote{Corresponding author. Email: s.majhi@gwu.edu}}
\affil[1]{\small Data Science Program, The George Washington University, 2121 I Street NW, Washington, DC 20052, USA (s.majhi@gwu.edu; abigailstein@gwmail.gwu.edu)}
\author[2]{Atish Mitra}
\affil[2]{\small Department of Mathematical Sciences, Montana Technological University, 1300 West Park Street, Butte, MT 59701, USA (amitra@mtech.edu)}
\affil[3]{\small Department of Petroleum Engineering, Montana Technological University, 1300 West Park Street, Butte, MT 59701, USA (bkoenig@mtech.edu, btodd@mtech.edu)}
\author[1]{Abigail Stein}
\author[3]{Burt Todd}

\title{Unsupervised Topological Characterization of Slug, Churn and Annular Flow in Small-Diameter Vertical Pipes Using Euler Characteristic Surfaces}

\date{}

\begin{document}

\maketitle

\begin{abstract}
Regime transitions in vertical gas--liquid flow are read from
mechanistic models or assigned by eye, and for churn flow neither
route has produced a settled criterion. We ask how much an
unsupervised topological description of flow video can establish about
them. The description is the Euler characteristic surface, a count of
the connected gas and liquid regions in each frame across
morphological scales. Kernels derived from it are combined by
unsupervised multiple kernel learning, and the trials are divided into
contiguous groups along gas velocity. For \nTrials{} air--water trials
in 2-in.\ tubing at one liquid rate, with three viewing stations
averaged, four groups are separated at \cutFourA, \cutFourB{} and
\cutFourC\,m/s. Each division falls at or within one trial of a change
in the visual evaluation of the same videos; cuts placed at random
would match as closely in \chanceCuts{} of cases. The divisions move
between viewing stations and under resampling, a kernel built on gas
velocity divides any sampled range into roughly equal parts and cannot
locate a transition, and the data do not support a revised transition
criterion. On independent images from an annulus, churn shows about
three times the across-scale variance of slug. Within the band the
published criteria assign to churn, the two upper cuts fall at graded
changes in the rankings, consistent with a progression within churn;
if confirmed by further data, they would subdivide the churn region.
\end{abstract}


\section{Introduction}
\label{sec:intro}

\subsection{Flow-regime transitions and the status of churn flow}
\label{sec:intro_churn}

Upward gas--liquid flow in a vertical pipe passes from slug flow through
churn flow to annular flow as the gas rate rises at a fixed liquid rate.
Slug flow is organized around Taylor bubbles that nearly fill the cross
section and are separated by liquid slugs. Annular flow has a continuous
gas core and a liquid film on the wall. Churn flow lies between them:
the liquid film thickens, reverses and bridges the pipe intermittently,
and no structure spans the cross section for
long~\cite{jayanti1992prediction,govan1991flooding,hewitt2012churn}.
Early work treated churn as an entrance
phenomenon~\cite{taitel1980,dukler1986}; later work established that it
persists along long pipes~\cite{shoham2006,waltrich2013}.

The slug-to-churn transition has been attributed to flooding of the
film around the Taylor bubble~\cite{govan1991flooding,jayanti1992prediction},
to the gas fraction in the liquid slug reaching a packing
limit~\cite{brauner1986,barnea1987}, and to the mean void fraction
reaching that of the Taylor-bubble section~\cite{mishima1984}. These
mechanisms give different
transition velocities for the same pipe. It is also debated whether
churn flow follows slug flow as the gas velocity rises or precedes it
along the pipe as an entrance effect: in a 25.4\,mm pipe 70 diameters
long, Holagh and Ahmed observed churn flow near the inlet giving way to
unstable slug flow by 60 diameters~\cite{holagh2025}. No unified
definition of churn flow has been established for upward
flow~\cite{holagh2025}; one has been proposed recently for downward
flow, from visualization in a 30\,mm pipe~\cite{arabi2026}.

Two terms are kept apart in this paper. A \emph{transition model} is a
criterion, derived from a physical mechanism, that gives the conditions
at which one regime gives way to another. A \emph{flow-pattern map} is
the chart that results when one or more transition models are drawn in
the plane of superficial gas and liquid velocities for a given pipe and
fluid pair. Wu et al.~\cite{wu2017critical} reviewed the available
transition models and recommended a set of them; we refer to those
criteria as the Wu et al.\ transition models, and to the chart they
produce for our pipe as the corresponding map.

\subsection{Why the transition matters in gas wells}
\label{sec:intro_loading}

In a gas well the produced gas must carry water and condensate to the
surface. As reservoir pressure falls the gas velocity drops, liquid
accumulates in the tubing, and production declines, a process known as
liquid loading~\cite{lea2003}. One remedy is a siphon string, a tube of
about 1\,in.\ diameter run inside the production tubing to raise the gas
velocity. Its design depends on which regime prevails in small-diameter
tubing at a given gas rate, and therefore on the transition models,
whose dependence on pipe diameter is still
debated~\cite{kaji2010,schlegel2009,ullmann2007}. This is the practical
motivation for studying transitions in 1-in.\ and 2-in.\
tubing~\cite{malin2019}.

\subsection{Identification of flow regimes from data}
\label{sec:intro_existing}

Most data-driven work on regime identification is supervised. Classifiers
have been trained on superficial velocities and fluid properties for
horizontal pipes~\cite{alhashem2019}, on differential-pressure signals in
a pipeline--riser system~\cite{xu2021}, on wire-mesh sensor images of
vertical flow~\cite{brantson2022}, and on camera images of vertical
flow~\cite{kadish2022}. A supervised classifier reproduces the labels it
was trained on. When those labels come from a transition model or from
visual judgement, the classifier cannot test either. An unsupervised
description of the flow does not have this limitation. It carries
one of its own: a division found without labels need not coincide with
a regime transition, and the evidence that it does has to be supplied
separately. Much of this paper is concerned with that evidence.

\subsection{Euler characteristic surfaces}
\label{sec:intro_tda}

Topological data analysis describes the shape of data by quantities
that are insensitive to small deformations. The Euler characteristic of
a binary image is the number of connected components of one phase minus
the number of the other. Computed for every frame of a video and for a
range of morphological dilation scales, it forms the Euler
characteristic surface (\ECS) introduced by Roy et
al.~\cite{roy2020,roy2023} and studied further
in~\cite{beltramo2021,Roy2025}. The \ECS is computed in near-linear time
with the Hoshen--Kopelman algorithm~\cite{hoshen1976}, and it is stable:
a bounded perturbation of the images produces a bounded change of the
surface in the $\Lone$ norm~\cite{Roy2025}. It has been applied to
drying droplets~\cite{roy2020,roy2023}, retinal
images~\cite{beltramo2021} and time series~\cite{Luwang2026}. It has not
previously been applied to pipe flow.

For a frame of gas--liquid flow the sign and size of the Euler
characteristic have a direct reading. Many small regions of one phase
dispersed in the other give a large magnitude; a few large connected
regions give a value near zero. How this reading changes with gas
velocity, and whether the change is abrupt enough to mark a transition,
is an empirical question.

\subsection{Objective and contributions}
\label{sec:intro_contributions}

The objective is to determine what an unsupervised topological
description of flow video establishes about regime transitions along a
gas-velocity sweep at fixed liquid rate, and how far the result depends
on choices made in acquiring and processing the video. The contributions
are as follows.
\begin{enumerate}[leftmargin=*, label=(\roman*)]
  \item A video-only procedure for dividing a velocity sweep into
        groups: two \ECS-derived kernels are combined by unsupervised
        multiple kernel learning, and the trials are partitioned into
        contiguous groups along gas velocity
        (Sections~\ref{sec:data_descriptors} and~\ref{sec:mkl}). Gas
        velocity orders the trials and enters nowhere else.
  \item A demonstration that a kernel built on gas velocity divides the
        sampled range into roughly equal parts whatever the flow does, so that
        it must be excluded from any procedure meant to locate a
        transition (Section~\ref{sec:results_velocity}).
  \item Divisions of \nTrials{} air--water trials in 2-in.\ tubing into
        two, three and four groups, each reported with the trials that
        bracket it, its support under resampling, and its position
        relative to six published transition criteria
        (Sections~\ref{sec:results_partition}
        and~\ref{sec:results_models}). The cuts interior to the band
        those criteria assign to churn fall at, or one trial from,
        graded changes in an observer's rankings of the same videos,
        consistent with a progression within churn
        (Sections~\ref{sec:results_malin} and~\ref{sec:results_models}).
  \item A measurement of how far the divisions move with viewing
        station, frame extraction, frame rate and normalization
        (Section~\ref{sec:results_sensitivity}).
  \item On an independent image set, a difference in across-scale
        variance between churn and slug images, and a measurement of how
        much a reported clustering accuracy depends on the way images
        are grouped (Section~\ref{sec:cross_dataset}).
\end{enumerate}

\section{The Euler characteristic surface of a flow video}
\label{sec:data_descriptors}

Scalar measurements such as pressure gradient and void fraction describe
how much of each phase is present. The Euler characteristic describes
how each phase is connected. Two frames with the same gas fraction
differ in Euler characteristic if the gas is in one region in the first
and in many in the second. Section~\ref{sec:ecs_theory} gives the
definitions, Section~\ref{sec:pipeline} the computation from video, and
Section~\ref{sec:distances} the two distances between trials used in
the rest of the paper. Figure~\ref{fig:ecs_multiscale} is a schematic of
the effect of morphological dilation on three idealized frames.

\subsection{Definitions and stability}
\label{sec:ecs_theory}

\begin{figure}[htb]
\centering
\begin{tikzpicture}[scale=1.0]
  \begin{scope}[xshift=0cm]
    \begin{scope}[yshift=5.5cm]
      \fill[white, draw=black, line width=0.6pt] (0,0) rectangle (2.5,2);
      \fill[black!85] (0.5,0.4) rectangle (2.0,1.6);
      \draw[black, line width=0.9pt] (0,0) rectangle (2.5,2);
      \node[left, font=\scriptsize] at (-0.1,1) {$s{=}0$};
      \node[below, font=\scriptsize] at (1.25,-0.25) {$\chi = 0$};
    \end{scope}
    
    \begin{scope}[yshift=2.7cm]
      \fill[white, draw=black, line width=0.6pt] (0,0) rectangle (2.5,2);
      \fill[black!85] (0.3,0.2) rectangle (2.2,1.8);
      \draw[black, line width=0.9pt] (0,0) rectangle (2.5,2);
      \node[left, font=\scriptsize] at (-0.1,1) {$s{=}3$};
      \node[below, font=\scriptsize] at (1.25,-0.25) {$\chi = 0$};
    \end{scope}
    
    \begin{scope}[yshift=0cm]
      \fill[black!85, draw=black, line width=0.6pt] (0,0) rectangle (2.5,2);
      \draw[black, line width=0.9pt] (0,0) rectangle (2.5,2);
      \node[left, font=\scriptsize] at (-0.1,1) {$s{=}8$};
      \node[below, font=\scriptsize] at (1.25,-0.25) {$\chi = {+}1$};
    \end{scope}
    
    \node[above, font=\small\bfseries] at (1.25,7.7) {Slug Flow};
    \node[below, font=\scriptsize, align=center] at (1.25,-0.6) 
      {Stable\\$\chi \approx 0$};
  \end{scope}
  
  \begin{scope}[xshift=3.5cm]
    \begin{scope}[yshift=5.5cm]
      \fill[black!85, draw=black, line width=0.6pt] (0,0) rectangle (2.5,2);
      \foreach \x/\y in {0.4/1.6, 1.0/1.6, 1.6/1.6, 2.2/1.6,
                         0.6/1.0, 1.4/1.0, 2.0/1.0,
                         0.4/0.4, 1.2/0.4, 2.2/0.4} {
        \fill[white] (\x,\y) circle (0.15);
      }
      \draw[black, line width=0.9pt] (0,0) rectangle (2.5,2);
      \node[left, font=\scriptsize] at (-0.1,1) {$s{=}0$};
      \node[below, font=\scriptsize] at (1.25,-0.25) {$\chi = {-}9$};
    \end{scope}
    
    \begin{scope}[yshift=2.7cm]
      \fill[black!85, draw=black, line width=0.6pt] (0,0) rectangle (2.5,2);
      \foreach \x/\y in {1.0/1.6, 2.2/1.6, 1.4/1.0, 0.4/0.4} {
        \fill[white] (\x,\y) circle (0.08);
      }
      \draw[black, line width=0.9pt] (0,0) rectangle (2.5,2);
      \node[left, font=\scriptsize] at (-0.1,1) {$s{=}3$};
      \node[below, font=\scriptsize] at (1.25,-0.25) {$\chi = {-}3$};
    \end{scope}
    
    \begin{scope}[yshift=0cm]
      \fill[black!85, draw=black, line width=0.6pt] (0,0) rectangle (2.5,2);
      \draw[black, line width=0.9pt] (0,0) rectangle (2.5,2);
      \node[left, font=\scriptsize] at (-0.1,1) {$s{=}8$};
      \node[below, font=\scriptsize] at (1.25,-0.25) {$\chi = {+}1$};
    \end{scope}
    
    \node[above, font=\small\bfseries] at (1.25,7.7) {Annular Mist};
    \node[below, font=\scriptsize, align=center] at (1.25,-0.6) 
      {Gas core\\$\chi < 0$};
  \end{scope}
  
  \begin{scope}[xshift=7cm]
    \begin{scope}[yshift=5.5cm]
      \fill[white, draw=black, line width=0.6pt] (0,0) rectangle (2.5,2);
      \fill[black!85] (0.3,1.4) -- (1.8,1.4) -- (1.9,1.8) -- (0.4,1.8) -- cycle;
      \fill[black!85] (0.4,0.7) rectangle (2.1,1.2);
      \fill[white] (0.9,0.8) rectangle (1.6,1.1);
      \fill[black!85] (1.6,0.2) -- (2.3,0.2) -- (2.3,0.6) -- (1.7,0.5) -- cycle;
      \draw[black, line width=0.9pt] (0,0) rectangle (2.5,2);
      \node[left, font=\scriptsize] at (-0.1,1) {$s{=}0$};
      \node[below, font=\scriptsize] at (1.25,-0.25) {$\chi = {+}1$};
    \end{scope}
    
    \begin{scope}[yshift=2.7cm]
      \fill[white, draw=black, line width=0.6pt] (0,0) rectangle (2.5,2);
      \fill[black!85] (0.2,1.3) rectangle (2.0,1.9);
      \fill[black!85] (0.3,0.5) rectangle (2.2,1.4);
      \fill[white] (0.7,0.7) rectangle (1.8,1.2);
      \fill[black!85] (1.4,0.1) rectangle (2.4,0.7);
      \draw[black, line width=0.9pt] (0,0) rectangle (2.5,2);
      \node[left, font=\scriptsize] at (-0.1,1) {$s{=}3$};
      \node[below, font=\scriptsize] at (1.25,-0.25) {$\chi = {-}1$};
    \end{scope}
    
    \begin{scope}[yshift=0cm]
      \fill[black!85, draw=black, line width=0.6pt] (0,0) rectangle (2.5,2);
      \draw[black, line width=0.9pt] (0,0) rectangle (2.5,2);
      \node[left, font=\scriptsize] at (-0.1,1) {$s{=}8$};
      \node[below, font=\scriptsize] at (1.25,-0.25) {$\chi = {+}1$};
    \end{scope}
    
    \node[above, font=\small\bfseries] at (1.25,7.7) {Churn Flow};
    \node[below, font=\scriptsize, align=center] at (1.25,-0.6) 
      {Fragmented\\$\chi$ varies};
  \end{scope}
  
  \draw[->,very thick,black!50] (10.3,6.5) -- (10.3,0.5);
  \node[right, font=\scriptsize, rotate=-90, black!70] at (10.6,3.5) {Morphological dilation};
  
\end{tikzpicture}
\caption{\textbf{Schematic of morphological dilation.} Idealized binary representations of three flow regimes at increasing morphological scales $s$. Each column shows one regime (slug, annular mist, churn) at three dilation levels: $s=0$ (original resolution), $s=3$ (moderate coarse-graining), and $s=8$ (strong coarse-graining). Black regions represent gas phase, white regions represent liquid. The Euler characteristic $\chi = N_{\text{black}} - N_{\text{white}}$ quantifies topological structure at each scale. \textbf{Slug flow} gives $\chi = 0$ at fine and moderate scales (one large gas region in the liquid) and $+1$ once the gas fills the frame. \textbf{Annular mist} gives strongly negative $\chi$ at fine scales (many liquid droplets in a continuous gas core), rising to $+1$ as dilation of the gas fills the droplets. \textbf{Churn flow} gives a $\chi$ that changes sign across scales, as fragmented gas regions merge into one that encloses liquid. Every connected region is counted, including those that touch the frame. Stacking $\chi_s(t)$ across scales $s$ and time $t$ yields the Euler Characteristic Surface $\mathbf{E} \in \mathbb{Z}^{T \times S}$, a multiscale descriptor capturing both spatial organization and temporal dynamics.}
\label{fig:ecs_multiscale}
\end{figure}

\begin{definition}[Euler Characteristic]
\label{def:ec}
For a finite cell complex $K$ of dimension $d$, the Euler characteristic is
\begin{equation}
  \EC(K) := \sum_{i=0}^{d} (-1)^i \,|\sigma_i|,
\end{equation}
where $\sigma_i$ denotes the set of $i$-dimensional cells of $K$. For a planar binary
image, $\EC(K) = N_b - N_w$ where $N_b$ and $N_w$ are the numbers of
black and white connected components, respectively.
\end{definition}

The Euler characteristic is a homotopy invariant: it is preserved under continuous
deformations of $K$. Its utility as a scalar summary of connectivity makes it
computationally cheap ($O(n)$ per image) and interpretable: $\EC > 0$ indicates
many isolated gas regions in connected liquid, while $\EC < 0$ indicates
connected gas with isolated liquid regions.

\begin{definition}[\ECS]
\label{def:ecs}
For a binary image sequence, the Euler Characteristic Surface is the map
\begin{equation}
  \EC : \{1,\ldots,S\} \times \{1,\ldots,T\} \to \mathbb{Z}, \quad
  (s,t) \mapsto \EC_s(t) = N_b(s,t) - N_w(s,t),
\end{equation}
where $N_b(s,t)$ and $N_w(s,t)$ are the number of connected components of black 
and white pixels at morphological dilation scale $s$ and time $t$. In practice, $\EC$ is 
represented as a matrix $\mathbf{E} \in \mathbb{Z}^{T \times S}$ computed via 
hexagonal morphological dilation and the Hoshen--Kopelman algorithm 
(Section~\ref{sec:pipeline}).
\end{definition}

\begin{theorem}[Stability of \ECS, {\cite{Roy2025}}]
\label{thm:stability}
Let $P, Q \subset \Rbb^2$ be two finite point sets and let $\mathbf{E}_P$, $\mathbf{E}_Q$
be their respective Euler Characteristic Surfaces. Then
\begin{equation}
  \norm{\mathbf{E}_P - \mathbf{E}_Q}_1 \;\leq\; C \cdot d_H(P, Q),
\end{equation}
where $d_H$ denotes the Hausdorff distance and $C > 0$ is a constant depending only
on the ambient dimension and the resolution of the grid.
\end{theorem}

Theorem~\ref{thm:stability} is the reason for measuring the distance
between surfaces in the $\Lone$ norm.

\subsection{Computation from video}
\label{sec:pipeline}

The surface is computed in five steps: frame extraction, binarization,
hexagonal morphological dilation, counting of connected components, and
assembly into a matrix with one row per frame and one column per scale
(Figure~\ref{fig:ecs_construction}).

\begin{figure}[htb]
\centering
\begin{tikzpicture}[scale=0.9]
  
  \node[above, font=\small\bfseries] at (1, 5.2) {Temporal Sequence};
  
  \foreach \t/\y in {1/4, 2/3, 3/2, 4/1} {
    \begin{scope}[yshift=\y cm]
      \fill[white, draw=black!60, line width=0.5pt] (0,0) rectangle (2,0.8);
      \ifnum\t=1
        \fill[black!80] (0.4,0.2) rectangle (1.6,0.7);
        \fill[white] (0.8,0.3) rectangle (1.2,0.6);
      \fi
      \ifnum\t=2
        \fill[black!80] (0.3,0.5) rectangle (1.7,0.8);
        \fill[black!80] (0.5,0.0) rectangle (1.5,0.6);
        \fill[white] (0.9,0.2) rectangle (1.1,0.4);
      \fi
      \ifnum\t=3
        \fill[black!80] (0.2,0.3) rectangle (1.8,0.8);
        \fill[white] (0.7,0.4) rectangle (1.3,0.7);
      \fi
      \ifnum\t=4
        \fill[black!80] (0.5,0.2) rectangle (1.7,0.7);
        \fill[white] (0.9,0.3) rectangle (1.3,0.6);
      \fi
      \draw[black, line width=0.7pt] (0,0) rectangle (2,0.8);
      \node[left, font=\scriptsize] at (-0.1,0.4) {$t{=}\t$};
    \end{scope}
  }
  
  \node[above, font=\small\bfseries] at (4.5, 5.2) {Multiscale ($t{=}2$)};
  
  \foreach \s/\y in {0/4, 5/3, 10/2, 15/1} {
    \begin{scope}[yshift=\y cm]
      \fill[white, draw=black!60, line width=0.5pt] (3,0) rectangle (5,0.8);
      \ifnum\s<6
        \fill[black!80] (3.3,0.5) rectangle (4.7,0.8);
        \fill[black!80] (3.5,0.0) rectangle (4.5,0.6);
        \fill[white] (3.9,0.2) rectangle (4.1,0.4);
      \fi
      \ifnum\s>4 \ifnum\s<11
        \fill[black!80] (3.2,0.4) rectangle (4.8,0.8);
        \fill[black!80] (3.4,0.0) rectangle (4.6,0.7);
      \fi\fi
      \ifnum\s>9
        \fill[black!80] (3,0) rectangle (5,0.8);
      \fi
      \draw[black, line width=0.7pt] (3,0) rectangle (5,0.8);
      \node[left, font=\scriptsize] at (2.9,0.4) {$s{=}\s$};
    \end{scope}
  }
  
  \node[above, font=\small\bfseries] at (8.5, 5.2) {ECS Matrix};
  
  \begin{scope}[xshift=7cm, yshift=1.5cm]
    \fill[blue!60] (0,0) rectangle (0.5,0.6);
    \fill[blue!55] (0.5,0) rectangle (1.0,0.6);
    \fill[blue!65] (1.0,0) rectangle (1.5,0.6);
    \fill[blue!58] (1.5,0) rectangle (2.0,0.6);
    
    \fill[blue!62] (0,0.6) rectangle (0.5,1.2);
    \fill[blue!57] (0.5,0.6) rectangle (1.0,1.2);
    \fill[blue!60] (1.0,0.6) rectangle (1.5,1.2);
    \fill[blue!63] (1.5,0.6) rectangle (2.0,1.2);
    
    \fill[green!45] (0,1.2) rectangle (0.5,1.8);
    \fill[green!50] (0.5,1.2) rectangle (1.0,1.8);
    \fill[green!48] (1.0,1.2) rectangle (1.5,1.8);
    \fill[green!52] (1.5,1.2) rectangle (2.0,1.8);
    
    \fill[red!50] (0,1.8) rectangle (0.5,2.4);
    \fill[red!48] (0.5,1.8) rectangle (1.0,2.4);
    \fill[red!52] (1.0,1.8) rectangle (1.5,2.4);
    \fill[red!50] (1.5,1.8) rectangle (2.0,2.4);
    
    \draw[black, line width=1.0pt] (0,0) rectangle (2,2.4);
    
    \node[below, font=\scriptsize] at (1,-0.25) {Time};
    \node[left, font=\scriptsize, rotate=90] at (-0.3,1.2) {Scale};
    
    \node[below, font=\normalsize] at (1,-0.6) {$\mathbf{E} \in \mathbb{Z}^{T \times S}$};
    
  \end{scope}
  
\end{tikzpicture}
\caption{\textbf{Construction of the Euler characteristic surface (schematic).} (\textit{Left}) Four consecutive binary frames show temporal variation in gas-liquid topology. (\textit{Center}) At time $t=2$, morphological dilation at four scales ($s=0, 5, 10, 15$) produces the Euler characteristic $\chi_s(t)$ at each scale. (\textit{Right}) The ECS matrix $\mathbf{E} \in \mathbb{Z}^{T \times S}$ assembles all $\chi_s(t)$ values. Heatmap colors indicate topology: blue regions (positive $\chi$, dispersed gas), green (intermediate), red (negative $\chi$, connected gas).}
\label{fig:ecs_construction}
\end{figure}

\paragraph{Frame extraction.}
The videos were recorded in slow motion on an iPhone~7 and are stored at
30 frames per second of playback. They were captured at
\captureFps{} frames per second~\cite{malin2019thesis}, and the recording
timestamps agree: within a pass of the camera many videos last longer
at playback speed than the interval before the next recording began,
which excludes the phone's lower slow-motion rate of
120 frames per second~\cite{majhi2026channel}. In the primary analysis one frame in
nine is used, one every \sampleMs\,ms of flow, and resolves fluctuations
up to \nyquistHz\,Hz. For the station-averaged surfaces the three
stations of a trial are cut to the length of the shortest of its videos
and averaged frame by frame after normalization, which leaves a median of
\framesShipped{} frames per trial; the videos are not synchronized, so
the average combines frames recorded at different times. To test whether the sampling affects the results,
all frames were also processed (Section~\ref{sec:results_sensitivity}).
The earlier version of this paper took the playback rate as the capture
rate and stated every interval eight times too long.
(The thesis also describes one minute of slow-motion video as about ten
seconds of flow, a factor of six; the timestamps support eight.)

\paragraph{Binarization.}
Each frame is converted to grey scale by averaging the three colour
channels, and a fixed threshold $\tau = 0.6$ on normalized intensity
separates dark from bright pixels. The threshold was set by inspection
of frames at low, intermediate and high gas rates. Automatic choices
based on Euler numbers exist~\cite{snidaro2003}; a fixed value is used
here so that every frame is binarized alike.

\paragraph{Hexagonal morphological dilation.}
The binary grid is converted to a hexagonal lattice by taking every
other cell in a brick-wall pattern, which removes the directional bias
of a square lattice~\cite{roy2020}. Dilation at level $s \geq 0$ extends
each dark pixel to all neighbours within $s$ lattice steps.

\paragraph{Euler characteristic at each scale.}
Connected components of dark pixels ($N_b$) and bright pixels ($N_w$)
are counted with the Hoshen--Kopelman
algorithm~\cite{hoshen1976}. The Euler characteristic at scale $s$ and
frame $t$ is
\begin{equation}
  \EC_s(t) = N_b(s,t) - N_w(s,t).
  \label{eq:ec}
\end{equation}

\paragraph{Surface and normalization.}
Stacking $\EC_s(t)$ over $s \in \{0, \ldots, S-1\}$ with $S=30$ and
over the frames gives the matrix $\mathbf{E} \in \mathbb{Z}^{T \times S}$.
Before trials are compared, each scale is centred and scaled by the
median and interquartile range of $\EC_s$ pooled over all frames of all
trials in the analysis. Scales whose pooled interquartile range is zero
are dropped; at those scales dilation has merged the frame into a single
region in almost every frame, and $\EC_s$ carries no information.
Of the 30 scales, \scalesKept{} are retained in the primary analysis.

An earlier version of this analysis normalized each scale to $[0,1]$
within each video, using that video's own minimum and maximum. We
replaced it because it is discontinuous at the largest scales. There
$\EC_s$ equals 1 in nearly every frame. If it is 1 in every frame the
column is constant and was set to 0; if it departs from 1 in a few
frames, the remaining frames sit at the maximum and the column is close
to 1. Two videos with the same large-scale topology could therefore
differ by the full range of the normalized variable, and this difference
dominated the distance between trials.
Section~\ref{sec:results_sensitivity} reports the results under both
normalizations.

\subsection{Two distances between trials}
\label{sec:distances}

Two distances are computed from the surfaces. Each distance matrix
$D^{(m)}$ is divided by its largest entry, giving
$\tilde{D}^{(m)}$ with entries in $[0,1]$.

\paragraph{\texorpdfstring{$L^1$}{L1} alignment distance.}
Let $\mathbf{E}_i$ and $\mathbf{E}_j$ be two normalized surfaces, with
$T_i$ and $T_j$ frames. The distance is the smallest mean absolute
difference per frame over a range of temporal offsets $k$:
\begin{equation}
  d_{\mathrm{ECS}}(\mathbf{E}_i, \mathbf{E}_j) =
  \min_{k \in [-k_{\max},\, k_{\max}]}
  \frac{1}{\ell_k}
  \sum_{t=\max(1,\,1-k)}^{\min(T_j,\,T_i-k)}
  \sum_{s=1}^{S}
  \bigl| E_i(t+k,\, s) - E_j(t,\, s) \bigr|,
  \label{eq:ecs_dist}
\end{equation}
where $\ell_k = \min(T_j, T_i - k) - \max(1, 1-k) + 1$ is the number of
overlapping frames and $k_{\max} = 25$ samples, or \windowS\,s of flow. Dividing by
$\ell_k$ makes the distance independent of the length of the videos.
This gives $D^{(1)}$.

\paragraph{Amplitude distance.}
The second distance compares summary statistics of the surface. For each
trial the mean, standard deviation, maximum and minimum of $\EC_s$ over
the frames are computed at each scale and concatenated into a vector
$\mathbf{a}_i \in \Rbb^{4S}$. The vectors are scaled by the median and
interquartile range across trials, and
\begin{equation}
  D^{(2)}_{ij} = \|\mathbf{a}_i - \mathbf{a}_j\|_2.
  \label{eq:amp_dist}
\end{equation}

\paragraph{Gas velocity is not a distance.}
An earlier version of this analysis used a third distance,
$|u_{gs,i} - u_{gs,j}|$, between the superficial gas velocities of two
trials. Section~\ref{sec:results_velocity} shows why it is excluded.
Gas velocity is used only to place the trials in order.

\section{Unsupervised metric learning and partitioning}
\label{sec:mkl}

The two distances of Section~\ref{sec:distances} are converted to
kernels and combined with weights learned from the data without labels.
The combined kernel is then used to divide the trials into contiguous
groups along gas velocity. Background on multiple kernel learning (\MKL)
is given in~\cite{gonen2011}.

\subsection{From distances to kernels}
\label{sec:mkl_kernels}

Each normalized distance matrix is converted to a positive-semidefinite
kernel by
\begin{equation}
  K^{(m)}_{ij} = \exp\!\left(-\frac{(\tilde{D}^{(m)}_{ij})^2}{\sigma_m^2}\right),
  \label{eq:heat_kernel}
\end{equation}
with bandwidth set by the median heuristic on the same normalized
distances:
\begin{equation}
  \sigma_m = \mathrm{median}\{\tilde{D}^{(m)}_{ij} : i < j,\,
  \tilde{D}^{(m)}_{ij} > 0\}.
  \label{eq:median_bandwidth}
\end{equation}
Because $\sigma_m$ and $\tilde{D}^{(m)}$ are on the same scale, the
kernel does not change when the units of the underlying measurement
change.

\begin{remark}[Bandwidth in the earlier version]
\label{rem:bandwidth}
The earlier version took the median in
Eq.~\eqref{eq:median_bandwidth} from the distances before
normalization and applied it to the normalized distances in
Eq.~\eqref{eq:heat_kernel}. The bandwidths were then several times
larger than any normalized distance, every kernel entry lay between
\kernelFloor{} and 1, and the weight given to each kernel depended on
the units of its distance. Expressing gas velocity in ft/s instead of
m/s moved the first division from \scSubmitted{} to
\scSubmittedFt\,m/s, on the earlier version's velocity scale. With the
consistent bandwidth of Eq.~\eqref{eq:median_bandwidth} and no trace
normalization, the three-kernel model of the earlier version gives the
gas-velocity kernel a weight of \betaUgsConsistent{} and the two
\ECS-derived kernels none; with trace normalization added they receive
at most 0.55 together.
\end{remark}

Each kernel is double-centred,
\begin{equation}
  K^{(m)} \leftarrow \left(I - \tfrac{1}{n}\mathbf{1}\mathbf{1}^\top\right)
                     K^{(m)}
                     \left(I - \tfrac{1}{n}\mathbf{1}\mathbf{1}^\top\right),
\end{equation}
and projected onto the positive-semidefinite cone by setting negative
eigenvalues to zero. It is then divided by its trace~\cite{gonen2011},
\begin{equation}
  K^{(m)} \;\leftarrow\; \frac{K^{(m)}}{\operatorname{tr}(K^{(m)})},
  \label{eq:trace_norm}
\end{equation}
so that the per-kernel losses defined below lie in $[0,1]$ and can be
compared. The combined kernel is
\begin{equation}
  K_\beta = \sum_{m=1}^M \beta_m K^{(m)}, \quad
  \beta \in \Delta^{M-1} := \bigl\{\beta : \beta_m \geq 0,\,\textstyle\sum_m \beta_m = 1\bigr\},
  \label{eq:blended_kernel}
\end{equation}
with $M = 2$ here: $\beta_1$ weights the alignment kernel and $\beta_2$
the amplitude kernel.

\subsection{Objective}
\label{sec:mkl_objective}

For a number of groups $k$, kernel $k$-means assigns the $n$ trials to
groups so as to maximize the similarity within groups:
\begin{equation}
  \max_{H \in \mathcal{H}_k} \;\mathrm{tr}(K_\beta H H^\top),
  \label{eq:kernel_kmeans}
\end{equation}
where $\mathcal{H}_k = \{H \in \mathbb{R}^{n \times k} : H^\top H = I_k\}$
and, if trial $i$ belongs to group $c$ of size $n_c$, row $i$ of $H$ has
the entry $1/\sqrt{n_c}$ in column $c$ and zero elsewhere. Equivalently
one minimizes
\begin{equation}
  L(\beta, H) := \mathrm{tr}(K_\beta) - \mathrm{tr}(K_\beta H H^\top)
  = \sum_{m=1}^M \beta_m \underbrace{\bigl[\mathrm{tr}(K^{(m)}) -
    \mathrm{tr}(K^{(m)} H H^\top)\bigr]}_{=:\, f_m(H)} .
  \label{eq:mkl_obj}
\end{equation}
The weights and the assignment are found together by minimizing
$L(\beta, H)$ over $\beta \in \Delta^{M-1}$ and $H \in \mathcal{H}_k$.
The objective is linear in $\beta$ for fixed $H$ and reduces to kernel
$k$-means for fixed $\beta$.

We do not fix $k$. The earlier version set $k=3$ to match the three
regimes of the flow-pattern map, which assumes what the analysis is
meant to test. Results are reported for $k = 2$, 3 and 4.

\subsection{Alternating optimization}
\label{sec:mkl_algorithm}

\noindent\textbf{$H$-step.}
For fixed $\beta$ the relaxed problem is solved by the $k$ leading
eigenvectors of $K_\beta$~\cite{dhillon2004}. The sign of each
eigenvector is fixed so that its entry of largest magnitude is positive, the rows of
the eigenvector matrix are normalized to unit length, and $k$-means on
the rows gives the assignment.

\medskip
\noindent\textbf{$\beta$-step.}
For fixed $H$, minimizing $\beta^\top \mathbf{f}(H)$ over the simplex
puts all weight on the kernel with the smallest loss. To keep both
kernels in play an entropy term is added,
\begin{equation}
  \min_{\beta \in \Delta^{M-1}} \;
  \beta^\top \mathbf{f}(H) + \lambda \sum_{m=1}^M \beta_m \log \beta_m,
  \label{eq:beta_step}
\end{equation}
whose solution is
\begin{equation}
  \beta_m = \frac{\exp(-f_m(H)/\lambda)}{\sum_{m'=1}^M \exp(-f_{m'}(H)/\lambda)}.
  \label{eq:softmin}
\end{equation}
Small $\lambda$ concentrates the weight on one kernel and large
$\lambda$ makes the weights equal. We use $\lambda = \lambdaVal$ and
report the divisions for $\lambda$ between 0.01 and 1
(Section~\ref{sec:results_partition}).

\begin{algorithm}[H]
\caption{Unsupervised MKL for Flow Regime Clustering}
\label{alg:mkl}
\KwIn{Base kernels $K^{(1)}, \ldots, K^{(M)} \in \mathbb{R}^{n \times n}$, number of clusters $k$, regularization $\lambda > 0$, tolerance $\epsilon > 0$}
\KwOut{Learned weights $\beta^* \in \Delta^{M-1}$, cluster assignment $H^* \in \mathbb{R}^{n \times k}$}
\BlankLine
Initialise $\beta^{(0)} \leftarrow (1/M, \ldots, 1/M)$ \tcp*{Uniform weights}
$t \leftarrow 0$\;
\BlankLine
\Repeat{$|\tilde{L}(\beta^{(t+1)}, H^{(t+1)}) - \tilde{L}(\beta^{(t)}, H^{(t)})| < \epsilon$}{
  \tcp{H-step: Spectral clustering with fixed $\beta^{(t)}$}
  $K_{\beta^{(t)}} \leftarrow \sum_{m=1}^M \beta^{(t)}_m K^{(m)}$\;
  $V \leftarrow$ top-$k$ eigenvectors of $K_{\beta^{(t)}}$\;
  $\tilde{V}_{i\cdot} \leftarrow V_{i\cdot} / \|V_{i\cdot}\|_2$ for $i=1,\ldots,n$ \tcp*{Row normalization}
  $\{c_1, \ldots, c_n\} \leftarrow$ k-means on rows of $\tilde{V}$ \tcp*{20 restarts, k-means++}
  Construct $H^{(t+1)} \in \mathbb{R}^{n \times k}$ with $H^{(t+1)}_{ic_i} = 1/\sqrt{n_{c_i}}$, else 0\;
  \BlankLine
  \tcp{$\beta$-step: Entropy-regularized weight update with fixed $H^{(t+1)}$}
  \For{$m = 1, \ldots, M$}{
    $f_m(H^{(t+1)}) \leftarrow \mathrm{tr}(K^{(m)}) - \mathrm{tr}(K^{(m)} H^{(t+1)} (H^{(t+1)})^\top)$\;
  }
  \For{$m = 1, \ldots, M$}{
    $\beta^{(t+1)}_m \leftarrow \dfrac{\exp(-f_m(H^{(t+1)})/\lambda)}{\sum_{m'=1}^M \exp(-f_{m'}(H^{(t+1)})/\lambda)}$ \tcp*{Softmin}
  }
  $t \leftarrow t + 1$\;
}
\BlankLine
$\beta^* \leftarrow \beta^{(t)}$;\quad $H^* \leftarrow H^{(t)}$\;
\Return{$(\beta^*, H^*)$}\;
\end{algorithm}

The two steps are repeated until the change in $\beta$ is below
$10^{-6}$ in Euclidean norm, for at most 50 iterations, from 20 random
initializations (5 for single stations and re-extractions, 3 for
resamples), and the run with the lowest objective is kept; each
$k$-means step itself takes the best of 50 starts.

\begin{proposition}[Monotonic convergence]
\label{prop:convergence}
Let $\tilde{L}(\beta, H) := L(\beta, H) + \lambda \sum_m \beta_m \log \beta_m$.
If an $H$-step is kept only when it does not increase $\tilde{L}$, the
sequence generated by the alternating $H$-step and $\beta$-step
satisfies
$\tilde{L}(\beta^{(t+1)}, H^{(t+1)}) \leq \tilde{L}(\beta^{(t)}, H^{(t)})$
for all $t$, and the sequence of objective values converges.
\end{proposition}

\noindent The proof is in Appendix~\ref{app:proofs}. The $\beta$-step is
exact and never increases $\tilde{L}$; the $H$-step is a spectral
relaxation followed by $k$-means and can. The runs reported here do not
apply the safeguard, so the proposition describes the safeguarded
variant, and the restarts guard against a poor local solution of a
problem that is not jointly convex. Stability of the
combined kernel under perturbation of the images, its metric structure
and a generalization bound are given in Appendix~\ref{app:theory}; none of them
is used in the results below.

\subsection{Contiguous partition along gas velocity}
\label{sec:boundary_inference}

The $H$-step assigns trials to groups without regard to gas velocity,
and the groups it returns need not be contiguous when the trials are
placed in order of $u_{gs}$. To obtain divisions of the velocity axis we
solve the same objective restricted to contiguous groups. Let the trials
be indexed in order of increasing $u_{gs}$. For cut positions
$1 < c_1 < \cdots < c_{k-1} \leq n$ define the groups
$G_r = \{c_{r-1}, \ldots, c_r - 1\}$ with $c_0 = 1$ and $c_k = n+1$. We
take
\begin{equation}
  (c_1^*, \ldots, c_{k-1}^*) \;=\;
  \operatorname*{arg\,max}_{c_1 < \cdots < c_{k-1}}
  \;\sum_{r=1}^{k}
  \frac{1}{|G_r|}\!\sum_{i,j \in G_r} K_{\beta,ij},
  \label{eq:contiguous}
\end{equation}
over all positions that leave at least two trials in every group. The
search is exhaustive. The weights $\beta$ are those learned by
Algorithm~\ref{alg:mkl} and are not changed at this step.

A cut between trials $c^*-1$ and $c^*$ is reported as the interval
\begin{equation}
  \bigl[\,u_{gs}^{(c^*-1)},\; u_{gs}^{(c^*)}\,\bigr],
  \label{eq:boundary_midpoint}
\end{equation}
which we call its bracket. Where a single number is needed we quote the
midpoint; the data place the cut no more precisely than the bracket,
whose width is the spacing of the velocity grid (two spacings where a
station lacks a trial).

Two properties of this procedure limit what it can show. First, it
returns $k-1$ cuts for any data, including data in which the surface
changes gradually with gas velocity and there is no transition to find.
Second, the ordering assumes that, at a fixed axial position and fixed
liquid rate, the regimes succeed one another monotonically in $u_{gs}$.
It says nothing about development along the pipe, where the same inlet
conditions can give different regimes at different
heights~\cite{holagh2025}. A cut is therefore a candidate, and
Sections~\ref{sec:results_partition}
to~\ref{sec:results_models} examine how well each candidate is
supported.

\section{Data}
\label{sec:data}

\subsection{Montana Tech Vertical Flow Loop}
\label{sec:mtvfl_data}

The Montana Tech Vertical Flow Loop (\MTVFL) has a transparent vertical
test section 14.35\,m long that takes tubing of nominal diameter 0.5,
0.75, 1 and 2\,in.\ (Figure~\ref{fig:mtvfl}). Air and water enter at
the bottom. The flow is filmed at three stations, 2.41, 6.73 and 11.02\,m above
the lower pressure transducer~\cite{malin2019thesis}, or about \LDbot,
\LDmid{} and \LDtop{} diameters downstream of it; the two pressure
transducers are 12.57\,m apart.

\begin{figure}[htb]
\centering
\includegraphics[width=90mm]{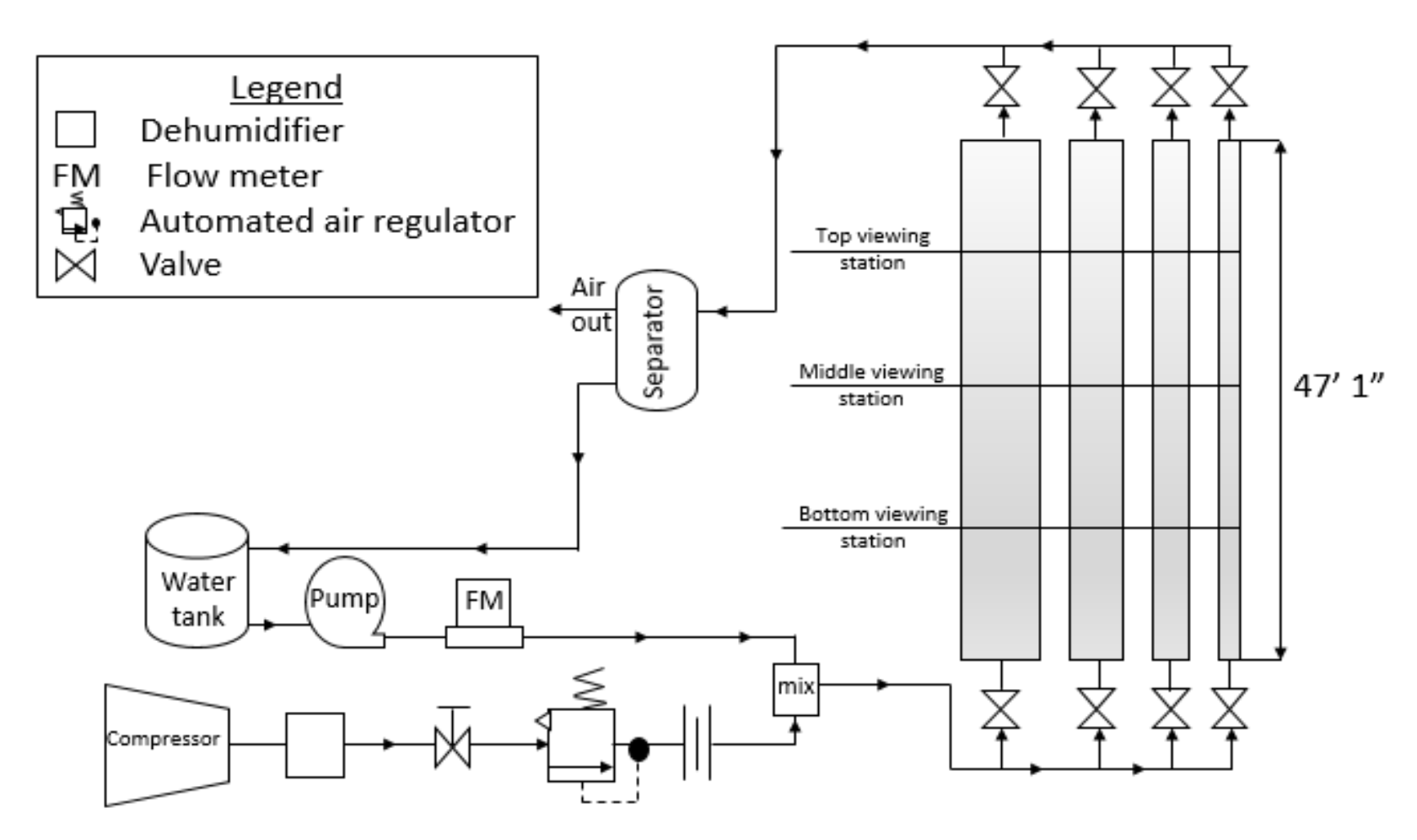}
\caption{Schematic of the Montana Tech Vertical Flow Loop, with the
three viewing stations and the air and water injection systems.}
\label{fig:mtvfl}
\end{figure}

The videos analysed here were recorded by M.~Malin in August 2019, in
the campaign described by Malin~\cite{malin2019thesis} and by Malin and
Todd~\cite{malin2019}, who evaluated them by eye. The earlier version of this paper cited that
evaluation as independent support; it is an evaluation of the same
videos. We use the series in 2-in.\ tubing: \nTrials{} trials at air
rates set from 86 down to 14 standard cubic feet per minute (SCFM,
setpoint volume at the meter's reference conditions) in steps of
2, with water held at 4 gallons per minute. The trials were run in that
order, without breaks, each recorded after the flow had been allowed to
stabilize. Each trial was filmed at the three stations with a
hand-held phone at a resolution of $480\times270$ pixels, one station
per pass of the camera up the loop. Surfaces are available for
\nTrials{} trials, \nStation{} of them at each station.

Superficial velocities are computed from the controller setpoints and
the bore of the schedule-80 tubing, $D = \boreMm$\,mm. The gas rate is
taken in situ, as in~\cite{malin2019thesis}: from a reference of 11.95\,psia
and 60\,$^\circ$F to the pressure at the viewing station, interpolated
linearly between the two transducers, and to the measured two-phase
temperature, both recorded for every trial~\cite{malin2019thesis}. Pressure
at the middle station runs from \pMidLo{} to \pMidHi\,psia across the
series. This gives $u_{ls} = \ulsVal$\,m/s and, at the middle station,
$u_{gs}$ from \ugsMin{} to \ugsMax\,m/s in mean steps of
\gridStep\,m/s; the velocities at the other two stations differ by the
pressure drop between them. The earlier version used the nominal
diameter of 50.8\,mm and took the setpoint volume at the meter's
reference conditions, without the pressure in the pipe, so the gas velocities here are
\velocityFactorLo--\velocityFactorHi{} times those it reported: its
slug--churn boundary of 9.08\,m/s, between the trials at 38 and
40\,SCFM, is \scSubmittedNow\,m/s on this scale. On the nominal
bore, this convention reproduces the velocities Malin gives for his
stage transitions to within \tableXXdev\,m/s.
The reference of 11.95\,psia, the local atmospheric pressure, is
Malin's; were the meter referenced to 14.696\,psia, every gas velocity
here would be 1.23 times higher. Only the comparison with velocity
criteria (Section~\ref{sec:results_models}) depends on this; the
comparison with Malin's rankings (Section~\ref{sec:results_malin}) does
not.

Figure~\ref{fig:frames} shows one frame at the middle station for nine
gas rates. The tube carries coloured reference marks and handwritten
labels, which are in the field of view in every frame.

\begin{figure}[tbp]
\centering
\includegraphics[width=\textwidth]{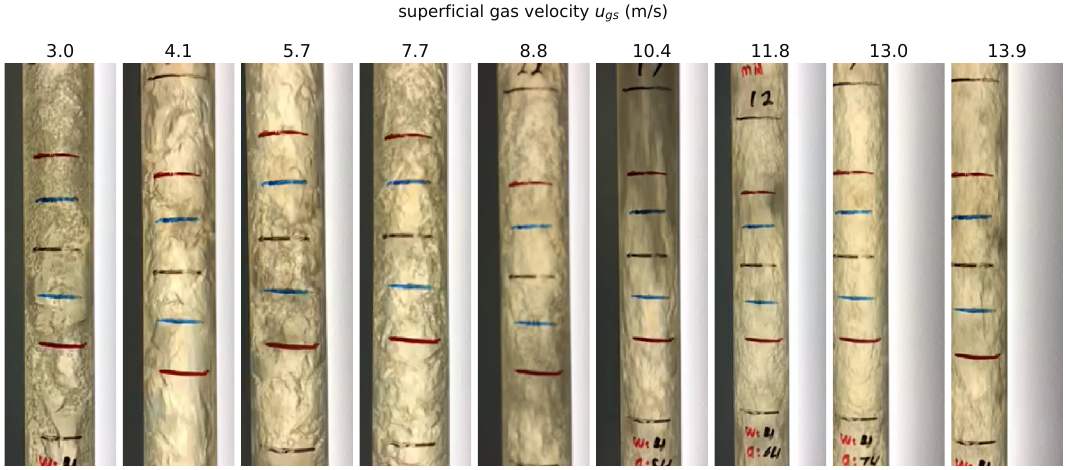}
\caption{One frame at the middle station for nine of the \nTrials{}
gas rates, 30\,s of playback (3.75\,s of flow) after the start of each
video. The horizontal marks
and the handwriting are on the outside of the tube.}
\label{fig:frames}
\end{figure}

The earlier version of this analysis labelled each trial by comparing
its $u_{gs}$ with the Wu et al.\ criteria and reported agreement with
those labels as an accuracy. Agreement with labels derived from a model
measures closeness to that model. It cannot show that the model is
wrong, and we do not report it. The only labels for these trials that
do not come from a model are Malin's visual rankings of the same
videos~\cite{malin2019thesis}, tabulated for every trial and station.
They are not independent of the video; they are independent of
this analysis, and Section~\ref{sec:results_malin} compares them with
the cuts.

\subsection{Reference transition models}
\label{sec:models}

The divisions found from the video are compared with six published
criteria, evaluated for air and water in the 2-in.\ tubing at
$u_{ls} = \ulsVal$\,m/s, in the forms given in the review of Wu et
al.~\cite{wu2017critical} and by Malin for this loop~\cite{malin2019thesis}.
Wu et al.\ found the first and the fourth the most reliable on their
compilation of data. With $u_m = u_{gs}+u_{ls}$,
$\Delta\rho = \rho_L-\rho_G$, $C_0 = 1.2 - 0.2\sqrt{\rho_G/\rho_L}$ and
$V = \sqrt{g D\,\Delta\rho/\rho_L}$, where $\rho_L$ and $\rho_G$ are the liquid
and gas densities, $\mu_L$ the liquid viscosity, $\sigma$ the surface
tension and $g$ the gravitational acceleration:
\begin{itemize}[leftmargin=*]
  \item \emph{Slug to churn, Barnea}~\cite{barnea1987}, with the
        liquid-slug void fraction of Brauner and Barnea~\cite{brauner1986}.
        The gas fraction in the liquid slug,
        \begin{equation}
          \alpha_s = 0.058\left[d_c
          \left(\frac{2 f_m u_m^3}{D}\right)^{0.4}
          \left(\frac{\rho_L}{\sigma}\right)^{0.6} - 0.725\right]^2,
          \qquad
          d_c = 2\left[\frac{0.4\,\sigma}{\Delta\rho\, g}\right]^{1/2},
        \end{equation}
        reaches 0.52~\cite{brauner1986}, with $f_m$ the Fanning friction
        factor at the mixture velocity, here the conventional
        $f_m = 0.046\,(\rho_L u_m D/\mu_L)^{-0.2}$.
        With this $f_m$ the condition can be solved for the mixture velocity,
        \begin{equation}
          u_m = 6.14\,D^{3/7}\left(\frac{\sigma}{\rho_L}\right)^{15/28}
          \left(\frac{\rho_L}{\mu_L}\right)^{1/14}
          \left(\frac{g\,\Delta\rho}{\sigma}\right)^{25/56},
        \end{equation}
        from which $u_{gs} = u_m - u_{ls}$.
  \item \emph{Slug to churn, Mishima and Ishii}~\cite{mishima1984}. The
        mean void fraction $\alpha = u_{gs}/(C_0 u_m + 0.35V)$ reaches
        that of the Taylor-bubble section,
        \begin{equation}
          \alpha_{\mathrm{TB}} = 1 - 0.813\left[\frac{(C_0-1)u_m + 0.35V}
          {u_m + 0.75V\left(g\,\Delta\rho\,D^3\rho_L/\mu_L^2\right)^{1/18}}\right]^{0.75}.
        \end{equation}
  \item \emph{Slug to churn, Taitel et al.}~\cite{taitel1980}. Churn is
        an entrance phenomenon, observed at height $l$ above the gas
        inlet when $l/D < 40.6\,(u_m/\sqrt{gD} + 0.22)$. The station
        heights are measured from the lower pressure transducer, and the
        distance from the inlet to the transducer is not recorded, so
        these values are lower bounds.
  \item \emph{Churn to annular, Mishima and Ishii}~\cite{mishima1984}.
        For a pipe narrower than
        $\sqrt{\sigma/(g\,\Delta\rho)}\,N_{\mu L}^{-0.4}\,[(1-0.11C_0)/C_0]^{-2}$,
        with $N_{\mu L} = \mu_L/[\rho_L\sigma\sqrt{\sigma/(g\,\Delta\rho)}]^{1/2}$,
        which is 59\,mm here, the transition is by flow reversal of the
        film,
        \begin{equation}
          u_{gs} = \left(\frac{g D\,\Delta\rho}{\rho_G}\right)^{1/2}(\alpha - 0.11),
        \end{equation}
        with $\alpha$ as above.
  \item \emph{Churn to annular, Taitel et al.}~\cite{taitel1980}:
        $u_{gs}\,\rho_G^{1/2} / [\sigma g\,\Delta\rho]^{1/4} = 3.1$.
  \item \emph{Churn to annular, flooding, Wallis}~\cite{wallis1969}:
        $u_{gs}\,[\rho_G/(g D\,\Delta\rho)]^{1/2} = 1$.
\end{itemize}
Table~\ref{tab:models} gives the values. On the nominal 2-in.\ bore, the
two criteria Wu et al.\ recommend give \nomWuSC{} and \nomWuCA\,m/s, against
Malin's 5.3 and 12.8\,m/s for this loop; the earlier version used 5.27
and 13.19\,m/s.

\begin{table}[H]
\centering
\caption{Transition velocities from published criteria for air and
water in 2-in.\ tubing at the liquid rate of the experiments, and the
cuts found from the video, with their brackets. Each criterion that
depends on the gas density is evaluated at the in-situ air density at
the middle station at its own gas velocity (1.23\,kg/m$^3$ for
Barnea's, 1.43\,kg/m$^3$ for that of Mishima and Ishii).
$^{a}$Beyond the highest gas rate of the sweep; evaluated at the density
of that rate. The entrance-length criterion is given at the
bottom, middle and top stations.}
\label{tab:models}
\small
\begin{tabular}{llc}
\toprule
Transition & Source & $u_{gs}$ (m/s) \\
\midrule
Slug to churn & Taitel et al.~\cite{taitel1980}, entrance length & 0.6, 2.1, 3.6 \\
Slug to churn & Mishima and Ishii~\cite{mishima1984} & 4.16 \\
Slug to churn & Barnea~\cite{barnea1987}, recommended by Wu et al.~\cite{wu2017critical} & 5.20 \\
Churn to annular & Mishima and Ishii~\cite{mishima1984}, recommended by Wu et al.~\cite{wu2017critical} & 12.87 \\
Churn to annular & Taitel et al.~\cite{taitel1980} & 13.23 \\
Churn to annular & Wallis~\cite{wallis1969}, flooding$^{a}$ & 17.44 \\
\midrule
Cut 1 of 3 & this work, four groups, stations averaged & 4.33 \;[4.13, 4.53] \\
Cut 2 of 3 & this work, four groups, stations averaged & 7.51 \;[7.32, 7.70] \\
Cut 3 of 3 & this work, four groups, stations averaged & 11.70 \;[11.58, 11.82] \\
\bottomrule
\end{tabular}
\end{table}

\subsection{Texas A\&M annulus images}
\label{sec:tamu_data}

The second data set is the public image set of Manikonda et
al.~\cite{manikonda2024vertical,manikonda2025dib}, recorded in a
vertical annulus formed by a 5.5-in.\ outer pipe and a 2.375-in.\ inner
pipe. The geometry differs from a circular pipe, and the structures
that can form in it differ accordingly; these images are used to test
the descriptor on data from another facility, and nothing about
transition velocities is inferred from them. The set contains
\tamuN{} still images, labelled by its authors as bubbly (278), slug
(285), churn (117) and Taylor bubble (267), taken under \tamuConds{}
operating conditions that differ in gas, injection pressure and
duration, liquid circulation rate and viewing floor. The article that
describes the set does not define the four labels. In the authors'
related experiments in the same tower, a Taylor bubble is the large
bubble at the front of a gas kick, and the flow behind it is described
separately, as bubbly, dispersed or churn flow depending on the
injection pressure and duration~\cite{obi2022kick}. Images from one run
often carry different labels, so a label describes what passes the
window in that frame; we use the labels as given. The images are stills
without a recorded time order, and a nominal gas injection rate is
listed for each run; no in-situ gas velocity is given.

Two versions of the features are used. In the first, each image is
processed exactly as the Montana Tech frames are, with the threshold
$\tau = 0.6$ and the whole image. In the second, the tube is located
in each image, the region is resampled to a common size, and the
threshold is $\tau = \tamuTau$. That threshold was chosen on these
images with the use of their labels, so the second version is adapted
to the data set and is not label-free.

\section{Results}
\label{sec:results}

\subsection{Surfaces of trials close in gas velocity are more alike}
\label{sec:results_trend}

Before any division is made we ask whether the distances carry
information about gas velocity at all. Table~\ref{tab:cuts} gives the
rank correlation between each distance and the difference in gas
velocity, over all pairs of trials, with a permutation test in which
the trials are reassigned to gas velocities at random. For the
station-averaged surfaces the correlation is \mantelEcs{} for the
alignment distance and \mantelAmp{} for the amplitude distance. It is
positive at each station and for the re-extracted videos, and in every
case $p \leq \mantelPmax$.

Figure~\ref{fig:embedding} plots the first principal component of the
combined kernel against gas velocity. It rises with gas velocity
(Spearman correlation \compRho{} for the station average), with scatter
between neighbouring trials that is comparable to the change across a
bracket. The component carries \compShare{} of the trace of the kernel,
so the structure is not dominated by one direction.

\begin{figure}[H]
\centering
\includegraphics[width=0.72\textwidth]{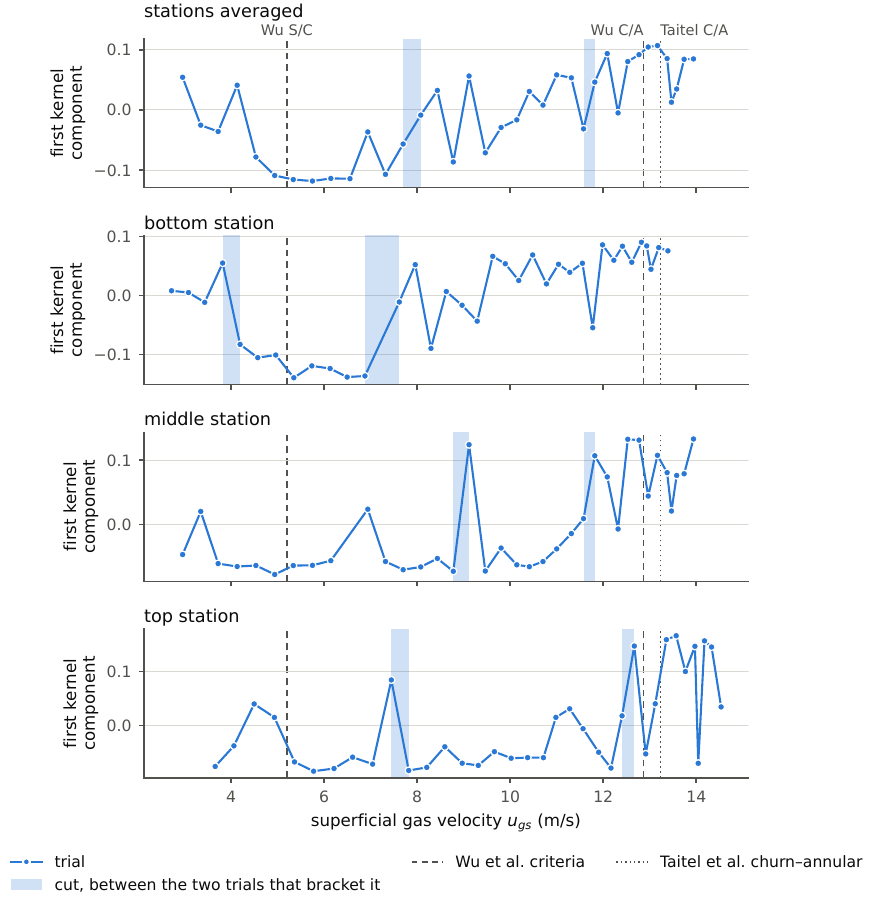}
\caption{First principal component of the combined kernel for each
trial, against superficial gas velocity, for the station-averaged
surfaces and for each station. Shaded bands are the brackets of the
cuts for three groups. Dashed lines are the slug--churn and
churn--annular criteria that Wu et al.\ recommend, computed for this
pipe (Table~\ref{tab:models}); the dotted line is the churn--annular
criterion of Taitel et al.\ at the local gas density.}
\label{fig:embedding}
\end{figure}

\subsection{Divisions of the station-averaged surfaces}
\label{sec:results_partition}

With the three stations averaged, as in the earlier version, the
procedure of Section~\ref{sec:boundary_inference} gives the following.
\begin{itemize}[leftmargin=*]
  \item Two groups: a cut at \cutTwo\,m/s.
  \item Three groups: cuts at \cutThreeA{} and \cutThreeB\,m/s.
  \item Four groups: cuts at \cutFourA, \cutFourB{} and \cutFourC\,m/s,
        with brackets [\cutFourAlo, \cutFourAhi], [\cutFourBlo,
        \cutFourBhi] and [\cutFourClo, \cutFourChi]\,m/s.
\end{itemize}
The learned weights for three groups are
$\beta = (\betaEcs, \betaAmp)$ for the alignment and amplitude kernels.
The cuts for three and four groups are the same for every $\lambda$
from 0.01 to 1. The groups returned by the unconstrained $H$-step are
not contiguous in gas velocity: along the velocity axis the label
changes \spectralChanges{} times. The contiguity of the groups is
imposed by Eq.~\eqref{eq:contiguous} and is not a finding.

Going from three groups to four keeps the upper cut, moves the cut at
\cutThreeA\,m/s down one bracket to \cutFourB\,m/s, and adds a cut near
\cutFourA\,m/s. The
cuts at \cutThreeA{} and \cutThreeB\,m/s are those the earlier version
reported as the slug--churn and churn--annular boundaries, at 9.08 and
14.67\,m/s on its velocity scale. They are
recovered here without a gas-velocity kernel. What has changed is their
interpretation, which is taken up in Section~\ref{sec:results_models}.

\paragraph{Support under resampling.}
Two resampling schemes give different answers
(Figure~\ref{fig:resampling}). When one trial is left out and the
whole procedure is repeated, each cut for four groups stays within one
grid step of its position in \jackFourMin{} to \jackFourMax{} of the
\nTrials{} cases. When the trials are resampled with replacement, which
on average omits about a third of them, a resample places a cut within
one grid step of \cutFourA, \cutFourB{} and \cutFourC\,m/s in
\bootNearA, \bootNearB{} and \bootNearC{} of \nBoot{} resamples. The
cuts are therefore insensitive to any single trial and sensitive to the
loss of several. With \nTrials{} trials on a grid of
\gridStep\,m/s, the data do not fix the position of any cut to within
a bracket.

\begin{figure}[H]
\centering
\includegraphics[width=0.72\textwidth]{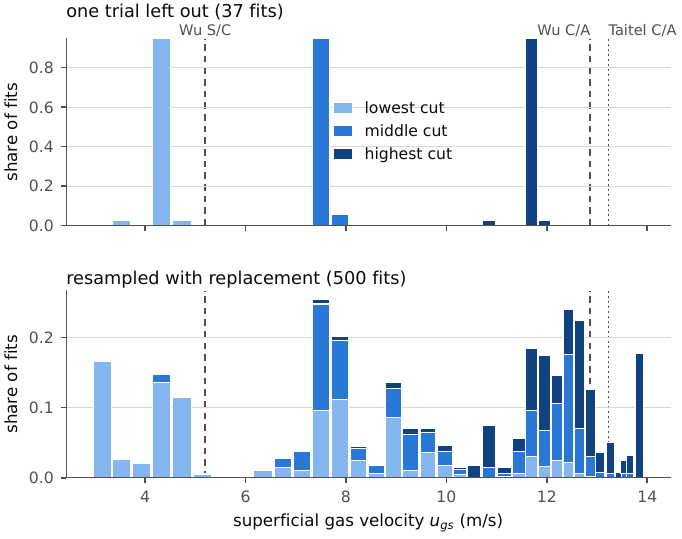}
\caption{Position of the cuts for four groups of the station-averaged
surfaces under two resampling schemes, as the share of resamples in
which a cut falls in each gap of the velocity grid.}
\label{fig:resampling}
\end{figure}

\subsection{Dependence on station, frame extraction, frame rate and normalization}
\label{sec:results_sensitivity}

Table~\ref{tab:cuts} and Figure~\ref{fig:cuts} repeat the analysis for
each station separately, for the same videos re-extracted with an
automatic crop of the tube, and for all frames at \captureFps{} per
second; Table~\ref{tab:cuts} also gives the result with the normalization
of the earlier version.

\begin{table}[H]
\centering
\caption{Divisions of the same trials in seven analyses. $r$ is the rank
correlation of each distance with the difference in gas velocity;
$\beta_1$ is the weight of the alignment kernel for three groups. Cuts
are midpoints of their brackets, in m/s. The last column gives, for each
cut for three groups, the percentage of leave-one-out fits in which it
stays within one grid step. The last row uses the normalization of the
earlier version.}
\label{tab:cuts}
\footnotesize
\setlength{\tabcolsep}{4pt}
\begin{tabular}{lccccccc}
\toprule
 & \multicolumn{2}{c}{$r$} & & \multicolumn{3}{c}{Cuts (m/s)} & Leave-one-out \\
\cmidrule(lr){2-3}\cmidrule(lr){5-7}
Analysis & align. & ampl. & $\beta_1$ & 2 groups & 3 groups & 4 groups & (\%) \\
\midrule
Stations averaged & 0.50 & 0.26 & 0.25 & 9.97 & 7.89, 11.70 & 4.33, 7.51, 11.70 & 95, 97 \\
\addlinespace
Bottom station & 0.47 & 0.31 & 0.25 & 9.46 & 4.00, 7.25 & 4.00, 7.25, 12.51 & 100, 97 \\
Middle station & 0.22 & 0.15 & 0.18 & 11.70 & 8.95, 11.70 & 11.16, 12.42, 12.87 & 50, 72 \\
Top station & 0.37 & 0.22 & 0.17 & 12.54 & 7.63, 12.54 & 5.15, 7.24, 12.54 & 8, 6 \\
\addlinespace
Re-extracted, 1 frame in 9 & 0.19 & 0.34 & 0.32 & 11.70 & 5.95, 12.65 & 5.95, 7.51, 12.65 & 41, 92 \\
Re-extracted, every frame & 0.19 & 0.37 & 0.26 & 8.26 & 8.26, 12.65 & 8.26, 11.70, 12.65 & 100, 78 \\
\addlinespace
Per-video normalization & 0.31 & 0.26 & 0.91 & 8.95 & 4.33, 7.89 & 4.33, 7.89, 10.28 & 100, 97 \\
\bottomrule
\end{tabular}
\end{table}

\begin{figure}[H]
\centering
\includegraphics[width=\textwidth]{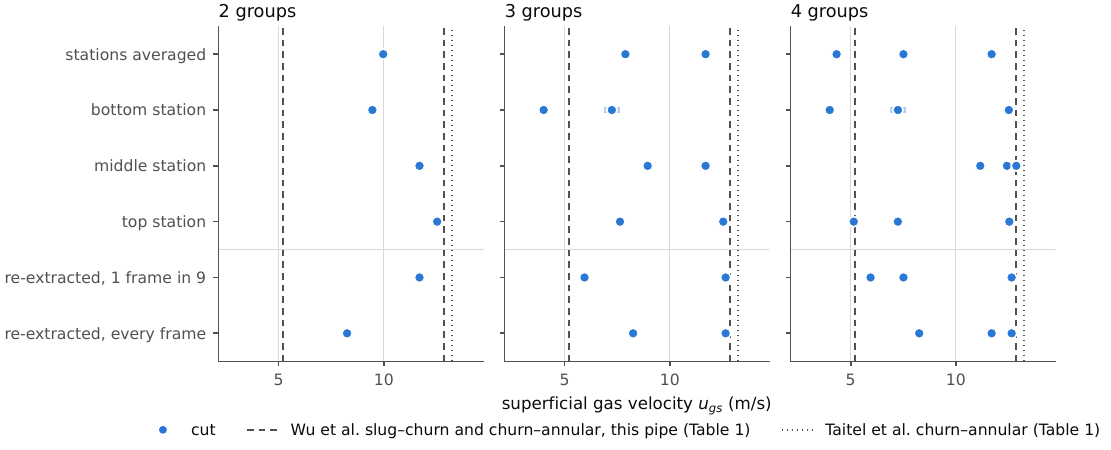}
\caption{Cuts for two, three and four groups in six analyses of the
same trials. Each point is the midpoint of a cut; its bracket is drawn behind it,
and shows only where it is wider than the marker. Dashed and dotted lines as in Figure~\ref{fig:embedding}.}
\label{fig:cuts}
\end{figure}

\paragraph{Station.}
The cuts differ between stations. For three groups the bottom station
gives \stBotThree, the middle station \stMidThree{} and the top station
\stTopThree\,m/s. A cut near \cutFourA\,m/s appears at the bottom
station, about \LDbot{} diameters downstream, and not at the middle
station. At the middle and top stations the cuts are also less stable: when
one trial is left out, the lower cut stays within one grid step in
\jackMidLow{} of cases at the middle station and \jackTopLow{} at the
top. The station average, which the earlier version analysed, therefore
combines positions at which the descriptor behaves differently. We do
not know whether this reflects development of the flow along the
pipe, as Holagh and Ahmed observed over the first 60 diameters of a
25.4\,mm pipe~\cite{holagh2025}, differences in how the three stations
were filmed, or both. The present data cannot separate them, because each
station was filmed once per trial.

\paragraph{Frame extraction.}
The surfaces of the earlier version were computed from frames cropped
by hand. Re-extracting the same videos with an automatic crop moves the
cuts for three groups from \cutThreeA{} and \cutThreeB{} to
\reThree\,m/s, and the correlation of the alignment distance with gas
velocity falls from \mantelEcs{} to \mantelEcsRe. The Euler
characteristic counts components over the cropped area, so it depends
on how much of the frame is kept, and with a hand-held camera the
fraction of the frame occupied by the tube varies from trial to trial.

\paragraph{Frame rate.}
Using every frame, at \captureFps{} per second, in place of one frame
every \sampleMs\,ms gives \reFullThree\,m/s for three groups, against
\reThree\,m/s at the lower rate from the same extraction. The two
differ in the lower cut by \reDiffLow\,m/s and agree in the upper.

\paragraph{Normalization.}
With the per-video normalization of the earlier version the
station-averaged cuts for three groups are \mmThree\,m/s, and the
alignment kernel receives a weight of \mmBeta. As explained in
Section~\ref{sec:pipeline}, that weight reflects a discontinuity of the
normalization at the largest scales.

Taken together, the lower cut for three groups lies between \lowMin{}
and \lowMax\,m/s and the upper between \upMin{} and \upMax\,m/s across
the six analyses with the revised normalization. The spread is several
times the grid spacing.

\subsection{A gas-velocity kernel divides the sampled range into roughly equal parts}
\label{sec:results_velocity}

The earlier version included a kernel built on $|u_{gs,i}-u_{gs,j}|$
and reported that this kernel alone placed the first division at
\velSubmitted\,m/s, which it read as evidence that the transition lay
well above the model value. Table~\ref{tab:range} and
Figure~\ref{fig:range} show what that kernel does. It contains no
information from the video. When the trials are restricted to a part of
the velocity range, its cuts move to the thirds of that part, each
within two trials; the in-situ velocities are not equally spaced,
because the pressure in the pipe rises with the gas rate, and the cuts
sit where the kernel's objective is lowest on that uneven grid. On the
equally spaced setpoint grid of the earlier version, the three
partitions of the \nTrials{} trials into groups of 12, 12 and 13 had
the same objective, so whether its first cut fell at \velTieA{} and
the second at \velTieB{} was decided by rounding error.

\begin{table}[H]
\centering
\caption{Cuts for three groups when the trials are restricted to part
of the velocity range (station-averaged surfaces, m/s).}
\label{tab:range}
\small
\begin{tabular}{lcccc}
\toprule
Trials used & $n$ & Thirds of the range & Gas-velocity kernel & Video kernel \\
\midrule
All trials & 37 & 6.62, 10.28 & 6.75, 10.85 & 7.89, 11.70 \\
$\leq$ 76\,SCFM & 32 & 6.36, 9.76 & 6.35, 9.97 & 7.89, 12.42 \\
$\leq$ 66\,SCFM & 27 & 6.00, 9.04 & 5.95, 9.29 & 4.33, 7.89 \\
$\leq$ 56\,SCFM & 22 & 5.54, 8.12 & 5.54, 8.26 & 4.33, 7.89 \\
$\geq$ 24\,SCFM & 32 & 7.94, 10.94 & 8.26, 11.45 & 7.51, 11.70 \\
$\geq$ 34\,SCFM & 27 & 9.27, 11.61 & 9.63, 11.95 & 12.42, 12.87 \\
\bottomrule
\end{tabular}
\end{table}

\begin{figure}[H]
\centering
\includegraphics[width=\textwidth]{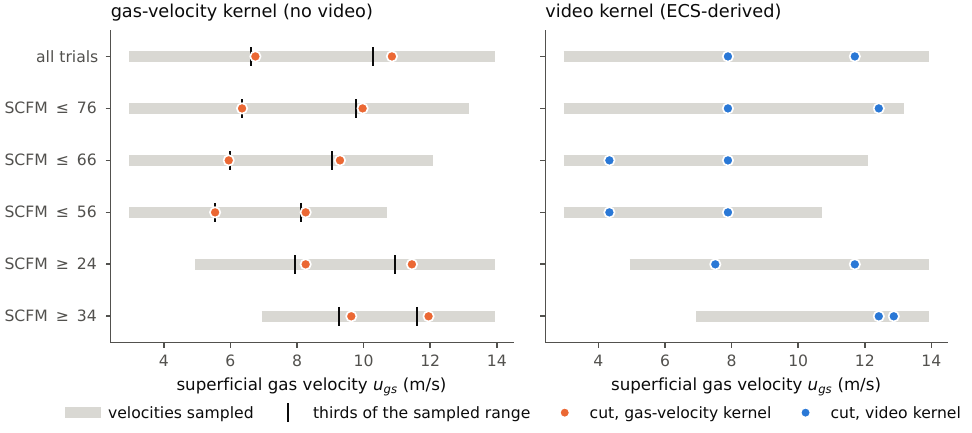}
\caption{Cuts for three groups when the analysis is restricted to part
of the velocity range. Left: a kernel built on gas velocity alone; the
cuts follow the thirds of whatever range is sampled. Right: the combined
video kernel.}
\label{fig:range}
\end{figure}

Any procedure that gives weight to such a kernel inherits this
behaviour, and the position of its cuts then depends on the range of
gas rates the experimenter chose. This is why gas velocity is excluded
from the metric. The video kernel does not follow the thirds of the
range; with three groups imposed its cuts are not fixed either:
removing the highest gas rates replaces the cut near \cutThreeB\,m/s by
the one near \cutFourA\,m/s. The number of groups and the range sampled both
affect which cuts are returned.

\subsection{Comparison with the visual evaluation of the same videos}
\label{sec:results_malin}

Malin~\cite{malin2019thesis} ranked every trial at every station by eye
in three categories: how far the liquid film falls back (film
reversal), how much of the liquid moves continuously upward, and the
length of the liquid slugs. His thesis tabulates the rankings, and he
defines four stages of flow by them. Because the rankings are given per
trial, they can be compared with the cuts as brackets of gas setpoint,
which does not require a convention for the gas velocity.
Table~\ref{tab:malin} gives the brackets in which his rankings reach
the values he uses to define the stages, and the cuts of this paper in
the same units.

The three cuts for four groups of the station-averaged surfaces fall at
\avCutsScfm\,SCFM. The lowest is the bracket in which, at the middle
station, Malin's film reversal reaches 3bc, a film that falls at least
3\,in.; he places the early stages of liquid-dominated slug flow
there. The middle cut is the bracket in which film reversal reaches
3abc. The upper cut is one trial below the bracket in which his
continuous upward flow reaches 2, his transition from stage~2 to
stage~3, at \malinCufTwoU\,m/s on our scale. His transition from
annular-mist to churn flow, film reversal reaching 3a, lies at
\malinAmc\,SCFM, above every cut.

This agreement has to be weighed against how often Malin's rankings
change. At the middle station one of them changes in \nChangeMid{} of
the \nBrackets{} brackets. Three distinct brackets drawn at random lie
at least as close to a change, in total, in \chanceCuts{} of draws, and
include the bracket of the slug onset in \chanceSlug. Station by
station the agreement is weaker. Malin's slug onset moves to lower
setpoints up the pipe, from \slugBot{} at the bottom station to
\slugTop{} at the top, as the in-situ gas velocity at a given setpoint
rises; the lowest cut for four groups moves the other way, from
\lowCutBot{} to \lowCutTop, and the middle station alone gives none of
the three station-averaged cuts. The station average is the configuration
the earlier version analysed, and it is the one that agrees; we do not know
why the stations taken singly do not.
The comparison was not planned: the cuts were computed before the
rankings were read. Neither Malin's evaluation nor this procedure is
blind to the order in which the trials were run.

\begin{table}[H]
\centering
\caption{Changes in Malin's visual rankings~\cite{malin2019thesis} and this
paper's cuts, as brackets of gas setpoint in SCFM: a bracket is the pair of
adjacent trials between which the change falls, read in order of decreasing
gas rate. Film reversal reaching 3a is Malin's stage~1/2 transition, which he
relates to the transition from annular-mist to churn flow; continuous upward
flow (CU) reaching 2 and 3 are his stage~2/3 and 3/4 transitions; film reversal
(FR) reaching 3bc, a film falling at least 3\,in., is where he places the early
stages of liquid-dominated slug flow. ``Changes'' counts the brackets, of
\nBrackets, across which any of his three rankings changes.}
\label{tab:malin}
\footnotesize
\setlength{\tabcolsep}{3pt}
\begin{tabular}{lcccccccc}
\toprule
 & \multicolumn{5}{c}{Malin's rankings} & & \multicolumn{2}{c}{This paper's cuts} \\
\cmidrule(lr){2-6}\cmidrule(lr){8-9}
Station & FR 3a & CU 2 & CU 3 & FR 3abc & FR 3bc & Changes & 3 groups & 4 groups \\
\midrule
Bottom & 78--80 & 68--70 & 56--58 & 40--42 & 22--24 & 12 & 20--22, 36--40 & 20--22, 36--40, 74--76 \\
Middle & 76--78 & 64--66 & 52--54 & 36--38 & 20--22 & 11 & 44--46, 62--64 & 58--60, 68--70, 72--74 \\
Top & 72--74 & 60--62 & 48--50 & 34--36 & 18--20 & 10 & 34--36, 64--66 & 22--24, 32--34, 64--66 \\
\addlinespace
Averaged & \multicolumn{6}{c}{(compare the middle station)} & 38--40, 62--64 & 20--22, 36--38, 62--64 \\
\bottomrule
\end{tabular}
\end{table}

\subsection{Comparison with transition models}
\label{sec:results_models}

Table~\ref{tab:models} sets the station-averaged cuts for four groups
beside velocity criteria from published models. Unlike the comparison
above, this one depends on how the gas velocity is evaluated
(Section~\ref{sec:mtvfl_data}). The lowest cut, with bracket
[\cutFourAlo, \cutFourAhi]\,m/s, lies \gapSC\,m/s \gapSCDir{} the
slug--churn criterion that Wu et al.\ recommend, that of Barnea
(\bbSC\,m/s), and close to that of Mishima and Ishii
(\miSC\,m/s). The middle cut, near
\cutFourB\,m/s, lies in the region that every model in the table
assigns to churn flow. The highest, with bracket [\cutFourClo,
\cutFourChi]\,m/s, lies \gapCA\,m/s \gapCADir{} the churn--annular
criterion that Wu et al.\ recommend, that of Mishima and Ishii, and
\gapCATaitel\,m/s \gapCATaitelDir{} that of Taitel et al.\ at the local
gas density; the flooding criterion of Wallis lies higher still.

We draw three conclusions, and they are narrower than those of the
earlier version.

First, the video gives no support for moving a transition criterion.
The earlier version identified the cut near \cutThreeA\,m/s on this
scale with the slug--churn transition and concluded that the model
under-predicted it. That identification came from imposing three
groups. With four, the procedure returns a cut close to the slug--churn
criterion of Mishima and Ishii, and \gapSC\,m/s below Barnea's, as well,
and the cut near \cutFourB\,m/s becomes a division
within the churn region. The $D^*$ argument of the earlier version does
not bear on this: for the 2-in.\ tubing $D^* \approx \Dstar$, below the
value of about 40 above which Taylor bubbles are
unstable~\cite{kataoka1987,hibiki2003,schlegel2009}, which shows that
slug flow can exist in this pipe and says nothing about the gas
velocity at which it ends.

Second, the station-averaged cuts lie near changes in the visual
evaluation of the same videos, the lowest at the onset of slug flow as
Malin saw it. Three cuts placed at random would do as well in
\chanceCuts{} of cases, and the stations taken singly do not agree.
The result is consistent with the descriptor responding to what an
observer sees in the film and slugs, and does not establish it.

Third, the upper cut lies \gapCA{} and \gapCATaitel\,m/s below the
churn--annular criteria of Mishima and Ishii and of Taitel et al., and
below the bracket in which Malin's rankings place
annular-mist to churn. We do not read it as that transition. It lies
one trial from his stage~2/3 transition, a change in how much of the
film is carried upward, within what the models call churn flow. The
trials were taken in order of decreasing gas rate, so gas velocity is
also confounded with time, and any change in lighting, camera position
or water condition during the campaign would appear as a change along
the velocity axis.

\paragraph{Structure within the churn band.}
The divisions at \cutFourB{} and \cutFourC\,m/s lie inside the
interval that every criterion of Table~\ref{tab:models} assigns to
churn flow, and they fall at graded changes in Malin's rankings: the
bracket in which film reversal reaches 3abc, and one trial from the
change in continuous upward transport. A graded character of churn flow has long been described
qualitatively, through flooding waves and oscillatory film
transport~\cite{hewitt2012churn,govan1991flooding}; the present
analysis adds a coordinate in which that progression registers without
labels or operator judgment. We report
the coincidence as consistency alone: three cuts placed at random
would agree as closely in \chanceCuts{} of cases, the agreement holds
for the station average and fails for the stations taken singly, and gas velocity is confounded with the order
in which the trials were run. If the coincidence survives independent
instrumentation and repeated filming (Section~\ref{sec:disc_settle}),
the divisions would refine the demarcation in the precise sense of
subdividing the churn region along this sweep, with the transition
boundaries left standing. The descriptor would then serve to
parameterize the progression within the churn region, a quantity
relevant to current discussions of how churn flow should be
defined~\cite{holagh2025,arabi2026}.

\section{An independent image set}
\label{sec:cross_dataset}

\subsection{Across-scale variance}
\label{sec:tamu_spatial}

For a single image the variance of $\EC_s$ over the 30 scales measures
how much the number of components changes as the image is coarsened. In
the Texas A\&M set this variance is higher for churn images than for
slug images: the ratio of means is \tamuRatio{} with the adapted
preprocessing and \tamuRatioAsIs{} with the Montana Tech preprocessing
(Figure~\ref{fig:tamu}, left).

Images taken under the same operating condition are not independent, so
the comparison is made between operating conditions. Averaging the
variance over the images of each condition and regime gives
\tamuChurnCells{} values for churn and \tamuSlugCells{} for slug; their
ratio of means is \tamuCellRatio{} and a Mann--Whitney test gives
$p = \tamuCellP$ with the adapted preprocessing
($p = \tamuCellPAsIs$ with the Montana Tech preprocessing). Only
\tamuBoth{} conditions contain both churn and slug images, and churn is higher in
\tamuBothHigher{} of them. The difference is present, and with
\tamuChurnCells{} churn conditions the evidence for it is modest.

\begin{figure}[H]
\centering
\includegraphics[width=\textwidth]{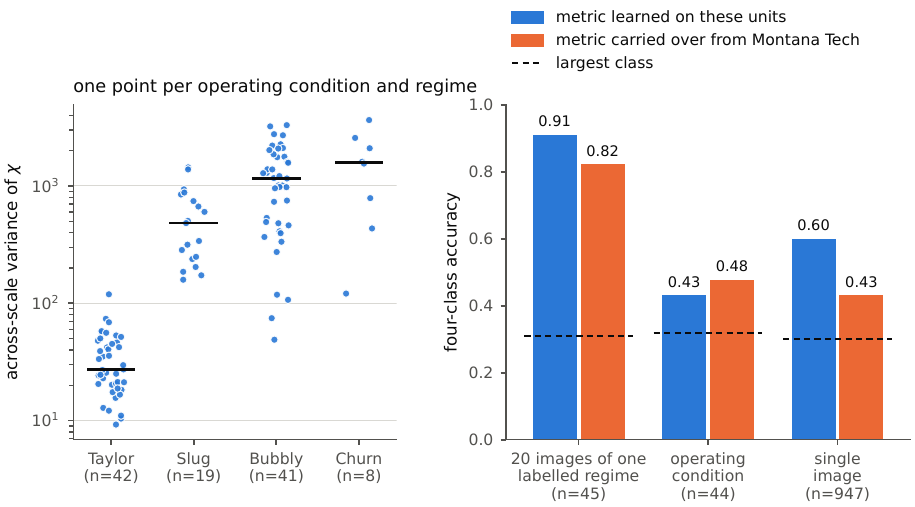}
\caption{Texas A\&M annulus images, adapted preprocessing. Left:
across-scale variance of $\EC$, one point for each operating condition
and regime; bars are medians. Right: accuracy of unsupervised
four-class clustering for three ways of forming the units that are
clustered.}
\label{fig:tamu}
\end{figure}

\subsection{Clustering, and what a reported accuracy depends on}
\label{sec:tamu_mkl}

The earlier version grouped the images into sets of 20 drawn at random
from within one labelled regime, clustered the sets, and reported an
accuracy of 95.6\%. Two features of that result need to be stated. The
sets were formed with the use of the labels, so every set is
homogeneous by construction and averaging over it removes
within-regime variation. And the features behind the number came from
files whose origin is not recorded; features regenerated from the
images give lower values.

Table~\ref{tab:tamu} reports four-class clustering with regenerated
features for three kinds of unit: sets of 20 images from one labelled
regime, as before; all images of one operating condition, which uses no
label to form the unit (conditions with a single image are left out, and
a unit is scored against the label most of its images carry); and single
images. For each, the metric is either learned on the
Texas A\&M units themselves, or carried over from Montana Tech, meaning
that the weights $\beta$ of Section~\ref{sec:results_partition} and the
bandwidths $\sigma_m$ of the Montana Tech analysis
(Section~\ref{sec:mkl_kernels}) are applied unchanged. In neither
case is the result a zero-shot transfer in the sense of applying a
fitted classifier: the two data sets do not share a label set, the
cluster assignment is always made on the Texas A\&M units, and the
adapted preprocessing was chosen with the labels.

\begin{table}[H]
\centering
\caption{Unsupervised four-class clustering of the Texas A\&M images
for three kinds of unit and two versions of the preprocessing. ARI is
the adjusted Rand index. ``Carried over'' applies the weights and
bandwidths learned on the Montana Tech trials. Churn recall is for the
metric learned on the units.}
\label{tab:tamu}
\footnotesize
\setlength{\tabcolsep}{3.5pt}
\begin{tabular}{llccccccc}
\toprule
 & & & Largest & \multicolumn{2}{c}{Learned on units} &
 \multicolumn{2}{c}{Carried over} & Churn \\
\cmidrule(lr){5-6}\cmidrule(lr){7-8}
Preprocessing & Unit & $n$ & class & Acc. & ARI & Acc. & ARI & recall \\
\midrule
Adapted & Sets of 20, one label & 45 & 31\% & 91\% & 0.80 & 82\% & 0.65 & 100\% of 5 \\
 & Operating condition & 44 & 32\% & 43\% & 0.12 & 48\% & 0.18 & 20\% of 5 \\
 & Single image & 947 & 30\% & 60\% & 0.28 & 43\% & 0.18 & 56\% of 117 \\
\addlinespace
Montana Tech & Sets of 20, one label & 45 & 31\% & 78\% & 0.45 & 62\% & 0.23 & 100\% of 5 \\
 & Operating condition & 44 & 32\% & 39\% & 0.04 & 45\% & 0.05 & 40\% of 5 \\
 & Single image & 947 & 30\% & 35\% & 0.05 & 35\% & 0.06 & 36\% of 117 \\
\bottomrule
\end{tabular}
\end{table}

With the adapted preprocessing, accuracy is \tamuAccSame{} for
label-homogeneous sets, \tamuAccCond{} for operating conditions and
\tamuAccImage{} for single images, against \tamuMajority{} for the
largest class. With the Montana Tech preprocessing applied unchanged the
three values are \tamuAccSameAsIs, \tamuAccCondAsIs{} and
\tamuAccImageAsIs. Carrying the metric over from Montana Tech gives
\tamuCarriedSame, \tamuCarriedCond{} and \tamuCarriedImage{} with the
adapted preprocessing. With the Montana Tech preprocessing and the
carried-over metric together, single-image accuracy is
\tamuAccImageAsIs{} against \tamuMajority{} for the largest class:
carried over unchanged, the pipeline transfers only weakly across
facility, geometry and image type.

The high value is obtained only when the units are formed from the
labels. For units formed without them the accuracy is between
\tamuHonestMin{} and \tamuHonestMax. Recall for churn is reported in
the table with the number of churn units, which is 5 for the first two
groupings; a recall of 100\% on five units is not strong evidence.
The earlier version's statement that the method matched or exceeded
supervised classifiers compared a number obtained with
label-homogeneous sets to numbers obtained on single images, and is
withdrawn.

\section{Discussion}
\label{sec:discussion}

\subsection{What the analysis establishes}

The Euler characteristic surface of a flow video changes systematically
with gas velocity. The relation is seen in both distances, at each
station, and in two independent extractions of the frames. A division
of the station-averaged surfaces into four groups gives cuts that stay
within one grid step in \jackFourMin{} or \jackFourMax{} of the
\nTrials{} leave-one-out fits and that lie at, or one trial from, changes
in the visual rankings of the same videos, the lowest where the observer
placed the onset of slug flow. In velocity, the lowest and highest lie
\gapSC{} and \gapCA\,m/s below the slug--churn and churn--annular criteria that Wu et al.\ recommend, and the middle one within the churn region. On an independent facility with a different geometry, churn
images have about three times the across-scale variance of slug images.

\subsection{What remains open}

It does not establish a transition velocity, and it does not establish
that the descriptor sees what the observer saw: cuts placed at random
agree as well with the visual rankings in \chanceCuts{} of cases, and the
stations taken singly do not agree. The cuts move by about 5\,m/s
between viewing stations and by about 2\,m/s with the way the frames
are cropped, and they are recovered in fewer than a third of bootstrap
resamples. It does not establish that the transition model is in error;
the conclusion of the earlier version to that effect is withdrawn. And
it does not establish that the descriptor identifies regimes on an
independent facility: when images are grouped without reference to
their labels, clustering accuracy is modest.

\subsection{Limitations}
\label{sec:disc_limitations}

\begin{itemize}[leftmargin=*]
  \item \emph{One liquid rate and one diameter.} The sweep covers one
        line in the plane of superficial velocities.
  \item \emph{No independent measurement.} There is no time-resolved
        pressure, void fraction or probe signal with which to identify
        the regime of a trial; only the mean pressures at the two
        transducers were recorded, and they are used only to evaluate
        the gas velocity.
  \item \emph{Run order.} Gas velocity decreases with trial number.
        Drift during the campaign cannot be separated from the effect of
        gas velocity.
  \item \emph{Acquisition.} The camera was hand-held, so the scale and
        position of the tube in the frame vary between videos, and the
        Euler characteristic is extensive in the area analysed. Marks
        and handwriting on the tube are in the field of view, and the
        handwriting differs between trials.
  \item \emph{Sampling.} The primary analysis uses one frame every
        \sampleMs\,ms of flow. The analysis of all frames changes one
        cut by \reDiffLow\,m/s.
  \item \emph{Number of groups.} The procedure returns cuts for any
        number of groups and cannot by itself decide between a sharp
        transition and a gradual change.
  \item \emph{Axial position.} The stations give different cuts, and
        the station average mixes them.
\end{itemize}

\subsection{What would settle the open questions}
\label{sec:disc_settle}

Three additions would turn the candidate divisions into tests. The
first is a fixed camera and a fixed region of analysis, which removes
the dependence on cropping. The second is a randomized order of gas
rates with repeated trials, which separates gas velocity from time and
gives a direct measure of repeatability. The third is a simultaneous
measurement that identifies the regime independently, such as a
conductance probe or a differential-pressure signal, recorded at each
station, together with visual rankings made blind to the gas rate. With these, the brackets reported here become hypotheses about
where the surface changes, and the comparison with transition models
becomes a test of the models.

\section{Conclusion}
\label{sec:conclusion}

We have applied Euler characteristic surfaces to video of upward
air--water flow in 2-in.\ tubing and asked what an unsupervised
analysis of them can establish about regime transitions.

The surface varies with gas velocity: trials closer in gas velocity
have more similar surfaces, with a rank correlation of \mantelEcs{} for
the station-averaged data. Dividing the \nTrials{} trials into four
contiguous groups gives cuts with brackets [\cutFourAlo, \cutFourAhi],
[\cutFourBlo, \cutFourBhi] and [\cutFourClo, \cutFourChi]\,m/s. In
the visual rankings of the same videos, the first is the bracket in
which the falling film indicates the onset of slug flow, the second a
further change in film reversal, and the third one trial from a change
in upward liquid transport; cuts placed at random would agree as well
in \chanceCuts{} of cases. In velocity, the first and third lie \gapSC{} and \gapCA\,m/s below the slug--churn and churn--annular criteria that Wu et al.\ recommend,
and the second within the churn region.

Each cut stays within one grid step in \jackFourMin{} or \jackFourMax{}
of the \nTrials{} leave-one-out fits. The cuts are not
robust to resampling several trials, to the choice of viewing station,
or to the way frames are cropped, and the lower cut for three groups
ranges from \lowMin{} to \lowMax\,m/s across the analyses. The limited data
therefore do not determine a transition velocity and give no basis for
revising the transition models.

What the data newly suggest lies within the churn region. The second and
third cuts fall inside the band that every criterion of
Table~\ref{tab:models} assigns to churn, at or one trial from graded
changes in the visual rankings of the same videos. This is consistent with a
progression within churn that a single regime label does not express,
and the descriptor provides a coordinate in which that progression can
be measured without labels or operator judgment. If confirmed, the
divisions would refine the demarcation by subdividing the churn
region, with the transition boundaries left standing.
Section~\ref{sec:disc_settle} states the additions that would turn
these candidate divisions into a test.

Two results concern method. A kernel built on gas velocity divides the
sampled range into roughly equal parts and cannot be used to locate a
transition. And on an independent image set, the accuracy of
unsupervised clustering is \tamuAccSame{} when the units clustered are
formed from the labels and \tamuAccCond{} when they are formed from the
operating conditions, so a reported accuracy cannot be interpreted
without the rule by which the units were formed.

\section*{Acknowledgements}

We thank Matt Malin, who ran the 2019 campaign, recorded every video
analysed here and made the visual rankings compared with the cuts in
Section~\ref{sec:results_malin}~\cite{malin2019thesis,malin2019}, and
Dave Rathgeber for the details of
the loop and the metering. The authors thank the staff of the Montana
Tech Vertical Flow Loop facility for experimental support and the
Montana Tech Department of
Petroleum Engineering for equipment access. The authors thank Anamika Roy for her explanation of the \ECS computational pipeline
of~\cite{Roy2025}, and the three reviewers, whose requests led to the
analyses of Sections~\ref{sec:results_sensitivity}
and~\ref{sec:results_velocity}.

\section*{CRediT authorship contribution statement}

\textbf{Brady Koenig:} Investigation, Data curation.
\textbf{Sushovan Majhi:} Conceptualization, Methodology, Formal analysis,
Software, Writing -- original draft.
\textbf{Atish Mitra:} Conceptualization, Methodology, Formal analysis,
Writing -- original draft.
\textbf{Abigail Stein:} Software, Data curation.
\textbf{Burt Todd:} Conceptualization, Investigation, Supervision, Resources.
All authors: Writing -- review \& editing.

\section*{Competing interests}

The authors declare no competing interests.

\section*{Funding}
This research did not receive any specific grant from funding agencies
in the public, commercial or not-for-profit sectors.

\section*{Data availability}

The Montana Tech videos analysed here are publicly available as
supplementary files to Malin's thesis~\cite{malin2019thesis}
(\url{https://digitalcommons.mtech.edu/grad_rsch/228/}). The Texas A\&M
image set is public (DOI: 10.17632/nxncbzzz38.3). The Euler
characteristic surfaces and the code that produces every number, table
and figure in this paper are available from the corresponding author on
reasonable request.

\appendix

\section{Properties of the combined kernel}
\label{app:theory}

\paragraph{Stability.}

\begin{proposition}[Combined kernel stability]
\label{prop:stability}
Let $P, Q$ be two sets of binary video frames with amplitude vectors
$\mathbf{a}_P, \mathbf{a}_Q$, and let $K_{\beta,P}$, $K_{\beta,Q}$ be
the corresponding combined kernels. Then
\begin{equation}
  \|K_{\beta,P} - K_{\beta,Q}\|_F \;\leq\;
    \beta_1 \cdot \frac{2C}{\sigma_1^2} \cdot d_{L^1}(E_P, E_Q)
    + \beta_2 \cdot L_{\mathrm{amp}} \cdot \|\mathbf{a}_P - \mathbf{a}_Q\|_1,
  \label{eq:stability_bound}
\end{equation}
where $C$ is the constant of Theorem~\ref{thm:stability},
$d_{L^1}$ is the alignment distance and $L_{\mathrm{amp}}$ is the
Lipschitz constant of the amplitude kernel.
\end{proposition}

\paragraph{Metric structure.}

\begin{theorem}[Blended kernel pseudo-metric]
\label{thm:metric}
Let $\mathcal{D}$ be the space of finite binary image sequences
(video trials) and let $K_\beta$ be the blended kernel defined in
Eq.~\eqref{eq:blended_kernel} with weights $\beta \in \Delta^{M-1}$.
Define the kernel-induced distance:
\begin{equation}
  d_{K_\beta}(P, Q) :=
    \sqrt{K_\beta(P,P) + K_\beta(Q,Q) - 2\,K_\beta(P,Q)}.
  \label{eq:kernel_distance}
\end{equation}
Then:
\begin{enumerate}[label=(\alph*)]
  \item $d_{K_\beta}$ is a pseudo-metric on $\mathcal{D}$ for every
        $\beta \in \Delta^{M-1}$ (non-negativity, symmetry, triangle
        inequality; identity of indiscernibles holds up to the kernel's
        null space).
  \item $d_{K_\beta}$ is Lipschitz continuous in $\beta$: for any
        $\beta, \beta' \in \Delta^{M-1}$ and trials $P, Q$,
        \begin{equation}
          \bigl|d_{K_\beta}(P,Q) - d_{K_{\beta'}}(P,Q)\bigr|
          \leq \sqrt{\|\beta - \beta'\|_1} \cdot \max_m \sup_{P,Q} d_{K^{(m)}}(P,Q).
        \end{equation}
  \item The embedding $P \mapsto \phi_\beta(P)$ into the reproducing
        kernel Hilbert space $\mathcal{H}_{K_\beta}$ satisfies
        $\|\phi_\beta(P)\|_{\mathcal{H}} = \sqrt{K_\beta(P,P)}$, which
        is bounded above by $\sqrt{\sum_m \beta_m K^{(m)}(P,P)}$.
\end{enumerate}
\end{theorem}

\paragraph{Generalization.}

\begin{theorem}[Generalisation bound for MKL clustering]
\label{thm:generalization}
Let $\mathcal{K} = \{K_\beta : \beta \in \Delta^{M-1}\}$ be the hypothesis
class of blended kernels with $M$ base kernels, each bounded:
$\sup_{i,j} K^{(m)}_{ij} \leq \kappa < \infty$.
Let $\hat\beta$ be the MKL solution on $n$ trials assumed i.i.d.\ drawn from
distribution $\mathcal{P}$ over $\mathcal{D}$ (see Remark~\ref{rem:iid} below),
and let $\beta^*$ be the population-optimal weight vector minimizing the expected
kernel k-means risk $\mathcal{R}(\beta) = \mathbb{E}_{P \sim \mathcal{P}}[L(\beta,
H^*_\beta)]$, where $H^*_\beta = \arg\min_{H \in \mathcal{H}_k} L(\beta, H)$ is
the population-optimal clustering for weight vector $\beta$.
Then with probability at least $1 - \delta$ over the draw of $n$ trials:
\begin{equation}
  \mathcal{R}(\hat\beta) - \mathcal{R}(\beta^*)
  \leq
  \frac{2\kappa\sqrt{2M\log(2/\delta)}}{\sqrt{n}}
  + \frac{4\kappa M}{n},
  \label{eq:pac_bound}
\end{equation}
where the first term is the Rademacher complexity contribution and the
second is the finite-sample correction.
\end{theorem}

\begin{remark}
\label{rem:iid}
The trials are gas rates on a fixed grid and are not independent draws
from a distribution, so the bound describes how the excess risk scales with
the number of trials and is not a guarantee for the present data. For
$M=2$, $n=\nTrials$, $\delta=0.05$ and $\kappa=1$ it evaluates to
\pacBound, which is vacuous.
\end{remark}

\section{Convex relaxation of MKL}
\label{sec:convex_mkl}

The alternating procedure in Section~\ref{sec:mkl} need not find a global
optimum. For completeness we state a convex relaxation of the problem, at
increased computational cost; it is not used in the analysis.

The kernel k-means objective can be written as a trace maximization:
\begin{equation}
  \max_{\beta \in \Delta^{M-1}} \;\max_{H \in \mathcal{H}_k}
  \;\mathrm{tr}\!\left(\sum_m \beta_m K^{(m)} H H^\top\right).
\end{equation}
Relaxing $H H^\top$ to the convex hull of rank-$k$ projection matrices,
\ie $H H^\top \in \mathcal{P}_k := \{Z : 0 \preceq Z \preceq I,\,
\mathrm{tr}(Z) = k\}$, gives a bilinear program in $(\beta, Z)$:
\begin{equation}
  \max_{\beta \in \Delta^{M-1},\; Z \in \mathcal{P}_k}
  \;\sum_m \beta_m \,\mathrm{tr}(K^{(m)} Z).
  \label{eq:convex_mkl}
\end{equation}
The objective is linear in $\beta$ for fixed $Z$, so its maximum over the
simplex is attained at a vertex, and for fixed $\beta$ the maximum over
$\mathcal{P}_k$ is the sum of the $k$ largest eigenvalues of
$K_\beta$ (Ky Fan's theorem). The relaxed problem is therefore solved
exactly by computing that sum for each kernel and keeping the kernel with
the largest; it puts all the weight on one kernel. The entropy term of
Section~\ref{sec:mkl} is what spreads the weight over several kernels, and
with it the problem has no such closed-form solution.

\section{Proofs}
\label{app:proofs}

\subsection{Proof of Proposition~\ref{prop:convergence} (monotonic convergence)}
\begin{proof}
At the H-step, spectral clustering finds $H^{(t+1)}$ by solving the continuous 
relaxation of kernel k-means exactly and then discretizing via k-means. Under the
safeguard of the proposition, $H^{(t+1)}$ is kept only if it satisfies $L(\beta^{(t)}, H^{(t+1)}) \leq L(\beta^{(t)}, 
H^{(t)})$; otherwise we retain $H^{(t+1)} = H^{(t)}$. Since $\beta^{(t)}$ is fixed 
at the H-step, the entropy term $\mathcal{E}(\beta^{(t)})$ is unchanged, so:
\begin{equation}
  \tilde{L}(\beta^{(t)}, H^{(t+1)}) \leq \tilde{L}(\beta^{(t)}, H^{(t)}).
  \label{eq:hstep_descent}
\end{equation}
At the $\beta$-step, $\beta^{(t+1)}$ exactly minimizes $\tilde{L}(\cdot, H^{(t+1)})$ 
over $\Delta^{M-1}$, since the entropy-regularized objective is strictly convex in 
$\beta$ and the softmin (Eq.~\ref{eq:softmin}) is its unique global minimizer. Thus:
\begin{equation}
  \tilde{L}(\beta^{(t+1)}, H^{(t+1)}) \leq \tilde{L}(\beta^{(t)}, H^{(t+1)}).
  \label{eq:betastep_descent}
\end{equation}
Chaining Eqs.~\eqref{eq:hstep_descent} and~\eqref{eq:betastep_descent}:
\begin{equation}
  \tilde{L}(\beta^{(t+1)}, H^{(t+1)}) \leq \tilde{L}(\beta^{(t)}, H^{(t)}),
\end{equation}
which is the monotonicity claimed. Since $L(\beta, H) \geq 0$ and the entropy
$\mathcal{E}(\beta) = -\sum_m \beta_m \log \beta_m \leq \log M$ for all
$\beta \in \Delta^{M-1}$, we have $\tilde{L}(\beta, H) = L(\beta,H) - \lambda
\mathcal{E}(\beta) \geq -\lambda \log M$, so the sequence is bounded below, and a
monotonically decreasing sequence bounded below converges.
\end{proof}

\subsection{Proof of Proposition~\ref{prop:stability}}
\begin{proof}
By the triangle inequality,
$\|K_{\beta,P} - K_{\beta,Q}\|_F
  \leq \sum_m \beta_m \|K^{(m)}_P - K^{(m)}_Q\|_F$.
For the alignment kernel, the mean-value theorem applied to
$K^{(1)}_{ij} = \exp(-(\tilde D^{(1)}_{ij})^2/\sigma_1^2)$ with
$\tilde D^{(1)}_{ij} \leq 1$ gives
$\|K^{(1)}_P - K^{(1)}_Q\|_F \leq (2/\sigma_1^2)\|\tilde D^{(1)}_P - \tilde D^{(1)}_Q\|_F$,
and by Theorem~\ref{thm:stability}
$\|\tilde D^{(1)}_P - \tilde D^{(1)}_Q\|_F \leq C \cdot d_{L^1}(E_P, E_Q)$.
The bound for $K^{(2)}$ follows from the smoothness of the heat kernel
on $\mathbb{R}^{4S}$.
\end{proof}

\subsection{Proof of Theorem~\ref{thm:metric} (blended kernel pseudo-metric)}
\begin{proof}
\textit{Part (a).} Since each $K^{(m)}$ is PSD and $\beta \in \Delta^{M-1}$,
the blended kernel $K_\beta = \sum_m \beta_m K^{(m)}$ is PSD (the PSD cone
is closed under non-negative linear combinations). Any PSD kernel induces a
pseudo-metric via Eq.~\eqref{eq:kernel_distance}~\cite{scholkopf2002}: 
non-negativity and symmetry are immediate; the triangle inequality follows 
from the fact that the map $P \mapsto \phi_\beta(P)$ into $\mathcal{H}_{K_\beta}$
satisfies $d_{K_\beta}(P,Q) = \|\phi_\beta(P) - \phi_\beta(Q)\|_{\mathcal{H}}$.
Identity of indiscernibles holds up to the null space of $K_\beta$: 
$d_{K_\beta}(P,Q) = 0$ iff $\phi_\beta(P) = \phi_\beta(Q)$ in $\mathcal{H}_{K_\beta}$.

\textit{Part (b).} Let $D_m = d_{K^{(m)}}(P,Q)$ and $D^* = \max_m \sup_{P,Q} D_m$.
Since $d_{K_\beta}(P,Q)^2 = \sum_m \beta_m D_m^2$, we have:
\[
  \bigl|d_{K_\beta}^2 - d_{K_{\beta'}}^2\bigr|
  = \Bigl|\sum_m (\beta_m - \beta'_m) D_m^2\Bigr|
  \leq \|\beta - \beta'\|_1 \cdot (D^*)^2.
\]
Using $|a - b| \leq \sqrt{|a^2 - b^2|}$ for $a, b \geq 0$ (since
$(a-b)^2 \leq (a+b)|a-b| = |a^2 - b^2|$):
\[
  \bigl|d_{K_\beta} - d_{K_{\beta'}}\bigr|
  \leq \sqrt{\|\beta - \beta'\|_1} \cdot D^*.
\]

\textit{Part (c).} By the reproducing property of $\mathcal{H}_{K_\beta}$,
$\|\phi_\beta(P)\|_{\mathcal{H}}^2 = \langle \phi_\beta(P), \phi_\beta(P)\rangle_{\mathcal{H}}
= K_\beta(P,P) = \sum_m \beta_m K^{(m)}(P,P)$,
so $\|\phi_\beta(P)\|_{\mathcal{H}} = \sqrt{K_\beta(P,P)} \leq \sqrt{\sum_m \beta_m K^{(m)}(P,P)}$,
with equality since $K_\beta(P,P) = \sum_m \beta_m K^{(m)}(P,P)$ exactly.
\end{proof}

\subsection{Proof of Theorem~\ref{thm:generalization} (generalization bound)}
\begin{proof}
The proof combines three standard tools; we verify the key condition
(bounded differences) and cite the remainder.

\textit{Step 1: Bounded differences.}
Define $\hat{\mathcal{R}}(\beta) = \frac{1}{n}\min_{H} L(\beta,H)$.
Fix $\beta$ and $H$, and replace trial $x_i$ with $x_i'$.
The term $\frac{1}{n}\mathrm{tr}(K_\beta)
= \frac{1}{n}\sum_j K_\beta(x_j,x_j)$ changes by at most $\kappa/n$
(only the $j=i$ diagonal entry is affected).
For $\frac{1}{n}\mathrm{tr}(K_\beta HH^\top)
= \frac{1}{n}\sum_c \frac{1}{n_c}\sum_{j,k \in c} K_\beta(x_j,x_k)$,
if trial $i$ belongs to cluster $c$ of size $n_c$, replacing $x_i$
affects $2n_c - 1$ entries each bounded by $\kappa$, weighted by
$1/(n \cdot n_c)$, so the change is at most $2\kappa/n$.
Since $|\min_H f - \min_H g| \leq \sup_H |f - g|$, for any $\beta$:
$|\hat{\mathcal{R}}(\beta; x_{1:n})
     - \hat{\mathcal{R}}(\beta; x_{1:n}^{(i)})| \leq 3\kappa/n$.

\textit{Step 2: Uniform convergence.}
The weight vector $\beta$ ranges over the $(M{-}1)$-simplex
$\Delta^{M-1}$, and $\hat{\mathcal{R}}(\beta)$ is Lipschitz in $\beta$
(since $|\hat{\mathcal{R}}(\beta) - \hat{\mathcal{R}}(\beta')| \leq
\kappa \|\beta - \beta'\|_1$).  A standard $\epsilon$-net argument
over $\Delta^{M-1}$ (covering number $\leq (3/\epsilon)^M$) combined
with McDiarmid's inequality at each net point and the Lipschitz
extension yields the uniform bound
$\sup_\beta |\hat{\mathcal{R}}(\beta) - \mathcal{R}(\beta)|
= O(\kappa\sqrt{M\log(1/\delta)/n})$
with probability $\geq 1-\delta$; see~\cite{bartlett2002}
and~\cite{gonen2011} for the analogous supervised MKL bound.

\textit{Step 3: Excess risk.}
Since $\hat\beta$ minimizes $\hat{\mathcal{R}}$:
$\mathcal{R}(\hat\beta) - \mathcal{R}(\beta^*) \leq
2\sup_\beta |\hat{\mathcal{R}}(\beta) - \mathcal{R}(\beta)|$,
giving the stated bound.
\end{proof}

\bibliography{main}

\end{document}